\documentclass[letterpaper, 10 pt, conference]{ieeeconf}  

\IEEEoverridecommandlockouts                              

\usepackage{nomencl}
\makenomenclature
\usepackage{amssymb}
\usepackage{amsmath}
\usepackage{graphicx}
\usepackage{xcolor}
\usepackage{subcaption}
\usepackage{hyperref}
\usepackage{algorithm,algcompatible}
\usepackage{algorithm}
\usepackage{algorithmicx}
\usepackage{algpseudocode}
\usepackage{algpseudocode}
\newcommand{\norm}[1]{\left\lVert#1\right\rVert} 

\newcommand{\R}{\mathbb{R}}
\newcommand{\C}{\mathcal{C}}
\newcommand{\M}{\mathcal{M}}
\newcommand{\T}{\mathcal{T}}

\newtheorem{theorem}{Theorem}

\newtheorem{lemma}{Lemma}
\newtheorem{remark}{Remark}
\newtheorem{proposition}{Proposition}

\title{\LARGE \bf
Betting for Sim-to-Real Performance Certificates}

\author{Yujia Chen$^1$, Bowen Weng$^1$
\thanks{$^{1}$Department of Computer Science, Iowa State University, Ames, IA 50011, USA
;{\tt\footnotesize yjchen@iastate.edu,bweng@iastate.edu}}
\thanks{This work was supported in part by financial assistance award 70NANB25H125 from U.S. Department of Commerce, National Institute of Standards and Technology, and in part by the National Science Foundation under Grant CPS-2550813.}
\thanks{The program that allows the exact reproduction of all figures in this paper can be found at the open-source code repository \href{https://github.com/ISUSAIL/Bet4Sim2Real-Certificate}{https://github.com/ISUSAIL/Bet4Sim2Real-Certificate}.}
}

\begin{document}

\maketitle

\begin{abstract}
Consider a typical \emph{test} of a robot system: one observes a sequence of outcomes concerning some aspect of interest (crash or no crash, tracking error, time to completion), and reports a mean (crash risk, average error, mean time to completion) and, more importantly, an interval guaranteed to contain that mean at a prescribed confidence, referred to as a \emph{performance certificate}. Given expensive real-world trials, the sample size is therefore small, and the certificate is often loose. Now consider the same procedure, except that before each real outcome is revealed, the operator ``peeks'' at a large bank of simulated results, and places a \emph{bet} on where the real outcome will land. As the real outcomes settle the bets, the operator gains or loses wealth. One's ``trust'' over simulators also shifts within the portfolio. This paper develops that idea into a sim-to-real betting certificate framework with three contributions: (i) An algorithm that links a scalable bank of simulators to effective bets, and the accumulated betting wealth to the certificate. (ii) A proof that the returned certificate is \emph{anytime valid}, covering the true mean with the prescribed probability, using any simulator bank. (iii) The guaranteed wealth-regret bounds yield configuration principles for the proposed algorithm and simulator bank design to deliver tight certificates. Experiments across synthetic distributions and real-world robot tests, covering both replayed standardized testing outcomes and online runtime evaluation, show the proposed method narrows the certificate by $51.6\%\pm16\%$ against classic and state-of-the-art baselines, and by $32.26\%\pm8\%$ in the extremely limited-sample regime ($\leq30$ samples).
\end{abstract}

\textbf{Notation: } The set of real numbers and positive integers are denoted by $\R$ and $\mathbb{Z}$, respectively. $\mathbb{Z}_k = \{1,\ldots,k\}$ for some $k\in\mathbb{Z}$. $\norm{\cdot}$ denotes the $\ell_2$-norm. For $A \subset \mathbb{R}$, $\text{hull}(A)=[\inf A, \sup A]$ (i.e., the closed interval hull).

\section{Introduction}\label{sec:intro}

Consider accurately, efficiently, and fairly evaluating a certain robot in the \emph{real} world, on \emph{real} hardware, and sometimes in \emph{real} time, concerning a particular aspect of its performance (e.g., risk estimation~\cite{chen2026betting,vincent24stochastic,weng26repeatability} and empirical reward evaluation of learned policies~\cite{peng18rand}). A fundamental tension often arises between \emph{what one wants}: a tight statistical bound that quantifies uncertainty about the performance measure at a prescribed confidence level, referred to as a \emph{performance certificate} (e.g., sufficiently tight lower and upper confidence bounds on risk), and \emph{what one can afford}: only a limited number of costly real-world tests from which to construct such a certificate. \textbf{The central idea of this paper is to leverage increasingly abundant and inexpensive \emph{simulators}, through the lens of \emph{betting}, to facilitate \emph{provably guaranteed} performance certification while reducing reliance on costly real-world tests}.

Formally speaking, consider a probability distribution $P$ over the measurable space $(\mathcal{X},\mathcal{F})$, representing the real-world data-generating distribution over the robot and environment conditions encountered during testing (e.g., the object pose in a pick-and-place task~\cite{nist24continuous,snyder26beyondbinary} and the friction and damping parameters affecting a navigation task~\cite{peng18rand}). Given a bounded performance scoring function $\psi:\mathcal{X} \rightarrow \M\subset\R$, we define the performance evaluation target as the scalar expectation (e.g., the expected success rate of a manipulator completing a pick-and-place task~\cite{badithela2026suresim,snyder26beyondbinary,chen2026betting}), along with its variance: 
\begin{equation}\label{eq:mustar}
    \mu^* := \mathbb{E}_{x \sim P}[\psi(x)],\ \sigma^2:= \mathbb{E}_{x \sim P}\big[(\psi(x)-\mu^*)^2\big].
\end{equation}
For the remainder of the paper, we take $\M\!\subseteq\![0,1]$. With proper normalization, this covers performance measures that are naturally bounded (e.g., risk and success rate), as well as measures that may be unbounded in principle but have justifiable limits under a specified testing envelope~\cite{weng26repeatability} (e.g., the tracking error of a robot operating within prescribed velocity and workspace limits).

\subsection{Probabilistic Certificate (without ``Simulators'')}\label{sec:intro:probci}
Given $\mu^*$ is unknown and cannot be observed directly, one approximation is through Monte Carlo sampling~\cite{hammersley1964monte,wasserman04zt}:
\begin{equation}\label{eq:mc}
    \hat{\mu}_t = \frac{1}{t}\sum_{i=1}^t \psi(x_i), \ x_i \overset{\mathrm{i.i.d.}}{\sim}P,
\end{equation}
whose empirical variance often appears in two forms: (i) the uncorrected $\hat{\sigma}_t^2:=\frac{1}{t}\sum_{i=1}^{t}[\psi(x_i)-\hat{\mu}_t]^2$, and (ii) the Bessel-corrected, unbiased sample variance $\hat{\sigma}_{t,u}^2:=\frac{1}{t-1}\sum_{i=1}^{t}[\psi(x_i)-\hat{\mu}_t]^2, t\ge 2$.

Although $\lim_{t\rightarrow \infty}\hat{\mu}_t\!=\!\mu^*$ (by the law of large numbers~\cite{hammersley1964monte,wasserman04zt}), $\hat{\mu}_t$ may substantively deviate from $\mu^*$ within the \emph{finite}, or even worse, extremely limited sampling regime, which is common in robot testing practice~\cite{vincent24stochastic,weng26repeatability,snyder26beyondbinary}.

This is where the \emph{certificate} (also referred to as a \emph{bound} interchangeably throughout this paper) fits in as an interval, $\C=[r_l,r_u]\subset\mathbb{R}$, $r_l, r_u \in \R$, $r_l \le r_u$, providing the range of candidate values for the unknown $\mu^*$. The \emph{width} is defined as $r_u - r_l$ as an indication of the \emph{tightness} of the certificate. 

In practice, the construction of $\C$ often proceeds in a sampling-based fashion as shown in~\eqref{eq:mc}, and hence often takes a \emph{probabilistic} form as
\begin{equation}\label{eq:certificate}
\mathbb{P}(\mu^* \in \C) \ge 1 - \alpha.
\end{equation}
In the remainder of this paper, unless specified otherwise, we take $\alpha = 0.05$ (i.e., a confidence level of $95\%$).

\subsubsection{Concentration inequality based certificate}\label{sec:intro:probci:concen}

Among various methods to construct a probabilistic certificate~\cite{hoeffding1963probability,maurer2009empirical,waudbysmith2022wsr,wasserman04zt}, one way immediately following~\eqref{eq:mc} is to apply a certain \emph{concentration inequality}~\cite{hoeffding1963probability,wasserman04zt,maurer2009empirical} that guarantees,
\begin{equation}\label{eq:concentration}
    \mathbb{P}(\mu^* \in [\hat{\mu}_t-\gamma,\hat{\mu}_t+\gamma]) \ge 1-\delta(\gamma,t), \forall \gamma>0.
\end{equation}
$\delta(\gamma,t)$ is a function specified by the particular inequality (e.g., the Hoeffding's~\cite{hoeffding1963probability}). To obtain a certificate with confidence level at least $1-\alpha$, it must satisfy $\delta(\gamma,t)\le\alpha$. 

The Hoeffding's inequality is one instance of~\eqref{eq:concentration}, using the known range of observations, sometimes the worst-case variance of $\psi(x_t)$. Setting $\delta(\gamma, t) = \alpha$ and solving for $\gamma$ yields $\gamma = \sqrt{\frac{\ln(2/\alpha)}{2t}}$. 

One other example, the empirical Bernstein inequality~\cite{maurer2009empirical,weng26repeatability}, additionally incorporates the empirical variance $\hat{\sigma}_t^2$, yielding $\gamma = \sqrt{\frac{2\hat\sigma_t^2 \ln(2/\alpha)}{t}}+ \frac{7\ln(2/\alpha)}{3(t-1)}$. 

\subsubsection{Hypothesis test based certificate}\label{sec:intro:probci:hypothe}
An alternative route to construct $\C$ is through \emph{hypothesis testing}. Rather than bounding the deviation of $\hat{\mu}$ from $\mu^*$, this approach evaluates individual candidate claims $\mu$ across $\M$. Specifically, for each candidate value $\mu \in \M$, one tests the null hypothesis $H_0:\mu^*=\mu$ by computing a method-specific \emph{test statistic} from the observed samples and comparing it against a threshold determined by the significance level $\alpha$. If the statistics exceeds the threshold, it means $H_0$ is rejected at significance level $\alpha$, hence the candidate value is excluded from the certificate. The certificate is then the set of all candidate values that are \emph{not} rejected at level $\alpha$, i.e.,
\begin{equation}\label{eq:certif}
\!\C\!=\!\text{hull}\!\big(\!\left\{\mu\!\in\!\M : H_0\!(\!\mu^*\!=\!\mu)\! \text{ not rejected at level } \alpha \right\}\!\big).
\end{equation}

In practice, the decision of whether rejecting $H_0$ can be made from two perspectives, often referred to as the (i) \emph{p-value} approaches~\cite{wasserman04zt} and the (ii) \emph{e-value} or e-process approaches~\cite{Ramdas25eprocess,waudbysmith2022wsr}.

\paragraph{P-value based hypothesis test}\label{para:pvalue} One classic $p$-value based instance is the $t$-test, which assumes the observed samples are drawn from normal distribution and rejects $H_0$ at the significance level $\alpha$ when the test statistic $\T_t(\mu)=|\hat{\mu}_t-\mu|/(\hat{\sigma}_{t,u}/\sqrt{t})$ exceeds the critical value $t_{1-\frac{\alpha}{2},t-1}$, the $1-\alpha/2$ quantile of Student's $t$ distribution with $t-1$ degrees of freedom~\cite{wasserman04zt}. So the accepted candidate fits $\mathbb{P}(\T_t < t_{1-\frac{\alpha}{2},t-1}) \geq 1-\alpha$, and it can be reformatted as a certificate (at significance level $\alpha$) as
$\C_t^{\text{$t$-test}} = \left[\hat\mu_t \!-\! t_{1-\frac{\alpha}{2},\, t-1}\! \frac{\hat{\sigma}_{t,u}}{\sqrt{t}}, \hat\mu_t\! +\! t_{1-\frac{\alpha}{2}, t-1}\! \frac{\hat{\sigma}_{t,u}}{\sqrt{t}}\!\right]$.

One extension of $t$-test, referred to as the sequential $t$-test, takes repeated, pre-specified examinations/checks of the data. Among different variants of the sequential $t$-test~\cite{wasserman04zt}, in this paper, the test is evaluated at every $t \geq 2$ with significance level $\alpha/(t(t+1))$. 

A final p-value-based baseline considered in this paper is the $z$-test, which replaces the corrected sample standard deviation $\hat{\sigma}_{t,u}$ with a prespecified constant $\sigma_0$, and replaces the Student's $t$ critical value $t_{1-\alpha/2,t-1}$ with the standard normal critical value $z_{1-\alpha/2}$. The constant $\sigma_0$ may be a known population standard deviation. In our implementation, we use the worst-case standard deviation $\sigma_0=0.5$. This yields the certificate
\begin{equation}\label{eq:ztest-C}
    \C_t^{\text{$z$-test}} = \Big[\hat{\mu}_t-z_{1-\alpha/2}\sigma_0/\sqrt{t},\ \hat{\mu}_t+z_{1-\alpha/2}\sigma_0/\sqrt{t}\Big].
\end{equation}

Note the methods mentioned above are valid only at the specific, pre-determined $t$ used to compute the statistic. These test methods are distribution-dependent and valid only at predefined sample size(s), so adding additional looks (e.g., the p-hacking~\cite{Ramdas25eprocess}) or adapting the stopping rule based on the data invalidates the guarantee.

\paragraph{E-value based hypothesis test}\label{para:evalue} To test the hypothesis $H_0$ without a fixed $t$, the e-value approach constructs a so-called \emph{e-process}~\cite{Ramdas25eprocess}. Concretely, for a candidate $\mu$ of $H_0$, we define the process
\begin{equation}\label{eq:eprocess}
    M_t(\mu) = M_{t-1}(\mu)(1 + \lambda_t(\mu)(\psi(x_t)-\mu)), M_0(\mu) = 1.
\end{equation}
Let $\mathcal{F}_{t-1}$ denote the information available from observations $x_1, \dots, x_{t-1}$, $\lambda_t(\mu) \in (-\frac{1}{1-\mu},\ \frac{1}{\mu})$ is a predictable factor depending on $\mathcal{F}_{t-1}$, representing the degree of belief about the relation between the next observation and the hypothesis, ensuring $M_t(\mu)>0$. 

The process also satisfies $\mathbb{E}[M_t(\mu) \mid \mathcal{F}_{t-1}] \le M_{t-1}(\mu)$ under $H_0$, making $M_t(\mu)$ a nonnegative supermartingale. By Ville's inequality~\cite{ville1939etude}, such a process satisfies 
\begin{equation}\label{eq:ville}
    \mathbb{P}\bigg( \sup_{t \geq 1} M_t(\mu) \geq 1/\alpha \Big| H_0 \bigg) \leq \alpha,
\end{equation} 
which means $H_0$ is rejected when $M_t(\mu) \ge 1/\alpha$. 

To reject a false $H_0$ as early as possible, we want $M_t(\mu)$ to grow as fast as possible. This objective can be understood through a \emph{betting} metaphor: one can take the process $(M_t(\mu))_{t \ge 0}$ as a betting process, where $M_t(\mu)$ represents the accumulated \emph{wealth} of a gambler sequentially betting against $H_0$, and $\lambda_t(\mu)$ is the \emph{stake} placed before observing $x_t$. At each round, the gambler receives a payoff $\lambda_t(\mu)(\psi(x_t) - \mu)$: positive when $\psi(x_t)$ deviates from $\mu$ in the direction predicted by the sign of $\lambda_t(\mu)$, and negative otherwise. Under $H_0$, the payoffs are mean-zero, hence wealth does not grow in expectation. When $H_0$ is false, wealth tends to grow. 

The stake $\lambda_t(\mu)$ determines how quickly evidence accumulates: faster wealth growth rejects a false candidate $\mu$ sooner, tightening $\C_t$ with fewer real-world trials. Following Kelly's criterion~\cite{kelly1956new}, if the true moments $(\mu^*,\sigma^2)$ were known, the true-moment oracle would use $\lambda^*(\mu)=\operatorname{clip}_{\delta}\{\kappa(\mu^*-\mu)/[\sigma^2+(\mu^*-\mu)^2]\}$ (see~\eqref{eq:delta-clip}), where $\kappa\in(0,1]$ controls betting aggressiveness. We refer to this oracle as ``ideal Kelly''~\cite{chen2026betting}. Since the true moments are unknown, Waudby-Smith and Ramdas (WSR)~\cite{waudbysmith2022wsr} replace them with predictable empirical estimates:
      $\lambda_t^{\mathrm{raw}}(\mu)
      =
      \kappa
      \frac{\hat{\mu}_{t-1}-\mu}
      {\hat{\sigma}_{t-1}^2+(\hat{\mu}_{t-1}-\mu)^2}$.
 
Furthermore, for the e-process to remain non-negative, the log-wealth analysis further requires every wealth factor to be strictly positive. One clips the raw stake as 
\begin{equation}\label{eq:delta-clip}
\lambda_t(\mu)\!=\! \text{clip}_\delta\big(\lambda^{\text{raw}}_t(\mu)\big) = \text{clip}\left(\lambda^{\text{raw}}_t(\mu), -\tfrac{1-\delta}{1-\mu}, \tfrac{1-\delta}{\mu}\right)
\end{equation}

\begin{remark}\label{rmk:e-certificate}[The analytical inversion of e-process]
    To \emph{analytically} invert Ville's rejection rule~\eqref{eq:ville} into the desired certificate of~\eqref{eq:certif}'s form, we evaluate $M_t(\mu)$ on a fine grid, identify the outermost non-rejected candidates, and refine the two outer rejection boundaries by bisection until each bracket falls below a prescribed tolerance. The resulting certificate is the interval hull
 $\mathcal{C}_t^{\text{e}}\!=\!
      \Big[\!\inf\!\{\mu\colon \sup_{t\geq1}M_t(\mu)<1/\alpha\}, 
       \sup\!\{\mu\colon \sup_{t\geq1}M_t(\mu)<1/\alpha\}\!\Big]$.
\end{remark}
\begin{figure}[t]
    \centering
    \includegraphics[width=\linewidth]{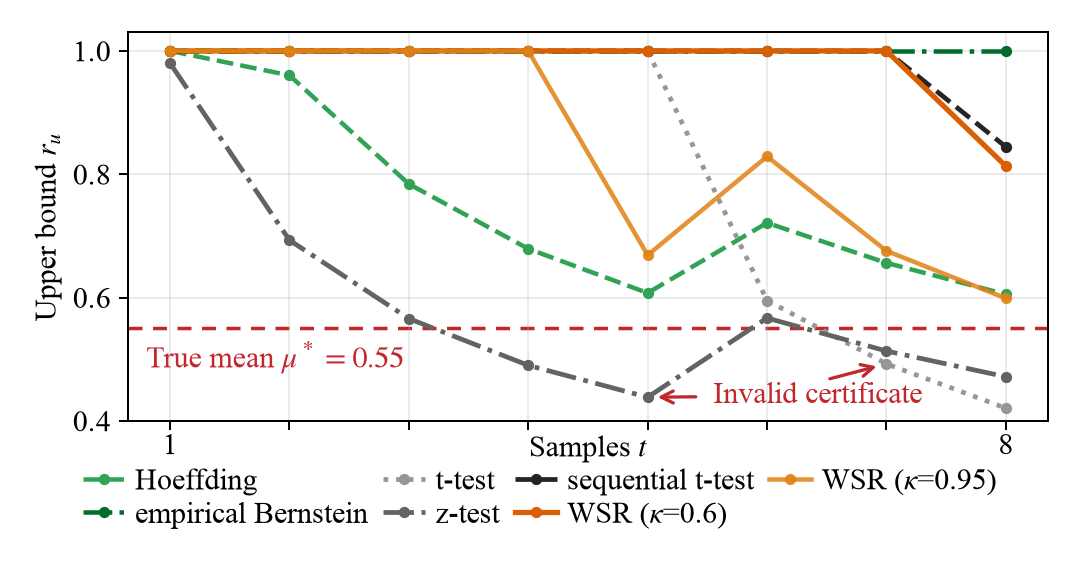}
    \vspace{-8mm}
    \caption{A sequence of $8$ sampled measures $\{0,0,0,0,0,1,0,0\}$ are generated from $x \sim \mathrm{Bernoulli}(0.55)$ (a toy example emulating a success rate certificate task of a peg-in-hole operation~\cite{nist24continuous}). The lower bounds of all certificates remain at zero across all sample sizes. The upper bound of each certificate is shown. Note p-value based methods occasionally have invalid coverage. All calculations are based on  Section~\ref{sec:intro:probci}.}
    \label{fig:example}
    \vspace{-5mm}
\end{figure}
Fig.~\ref{fig:example} summarizes the reviewed certificate methods in this section over a toy example with concrete values. 

Notably, all methods discussed above construct certificates \emph{solely} from real-world observations. Despite their different assumptions and trade-offs, certificate tightness, measured by the width of the resulting bounds, is fundamentally limited by the scarcity and cost of real-robot trials~\cite{liu2024curse,kalra2016driving}. This limitation has motivated a small but growing body of work~\cite{vincent24stochastic,luo25controlvariate,badithela2026suresim,chen2026betting,luo26x4val,mandyam26perry}, including this paper's proposal, to supplement real-world evidence with inexpensive and scalable robot simulations.

\subsection{Probabilistic Certificate (with ``Simulators'')}\label{sec:intro:sim2reali}
Consider a simulator-induced distribution $Q$ on the same measurable space as the real-world distribution $P$. A typical sim-to-real certification approach quantifies the discrepancy between $P$ and $Q$, then compensates for this gap when converting a certificate constructed from abundant samples of $Q$ into a bound on the real-world performance under $P$.

Vincent et al.~\cite{vincent24stochastic} \emph{assume} a one-sided sim-to-real gap and use simulation samples to construct a one-side upper bound $r_u$. To adapt to the certificate form of an interval as discussed in~\eqref{eq:certificate}, a lower bound $r_l$ can follow symmetrically under a separate reverse-gap assumption, as demonstrated in our released implementation. Consequently, the validity of Vincent et al. requires both assumed gaps to be correct, while conservative gap choices directly widen the resulting certificate. 

In contrast, Badithela et al.~\cite{badithela2026suresim} propose SureSim, which explicitly \emph{characterizes} the sim-to-real correction from \emph{paired} sim-real trials (e.g., moving the same red box from the same initial pose in simulation and reality and comparing task success or placement error). Such a ``pairing'' requirement is inherited from prediction-powered inference (PPI)~\cite{angelopoulos2023ppi}, where paired observations estimate the correction term and can reduce variance when simulation predictions are informative. However, obtaining such matched pairs can impose substantial practical constraints, particularly on robot testing, and may be impossible when no reliable real-to-simulation correspondence exists. A related method by Luo et al.~\cite{luo25controlvariate} likewise requires paired observations, using a control variate to reduce estimation variance. Their recent subsequent X4Val~\cite{luo26x4val} replaces the pairing with a learned surrogate, which still needs a prediction for every real trial, and is guaranteed only in the large-sample limit.

From a more fundamental aspect, the above approaches raise a natural simulator design question: what justifies the \emph{optimal} simulator for each approach (i.e., producing the tightest certificate for a fixed real-world testing budget)? For example, for Vincent et al., the optimal simulator would be the $Q$ whose gap from $P$ is identical to the assumed gap in the direction modeled by the bound. For SureSim, one favors a stable and predictable sim-to-real discrepancy, i.e., the optimal simulator is one for which the difference between paired real and simulated outcomes is constant across trials, or as nearly constant as possible. It is immediate that both views of optimality translates little into practical guidance for simulator design, selection, or calibration. Moreover, the approaches above rely on a \emph{single} simulation distribution $Q$ and scale primarily by drawing more samples from that same simulator (SureSim could explore multiple $Q$'s but only one has to be chosen for bias characterization). How to exploit a scalable family of parameterized simulators $\{Q_\theta:\theta\in\Theta\}$, including how to construct, select, and combine them for tighter certificates, therefore remains open. The proposed method is thus inspired.

\section{Main Method}\label{sec:main}
The proposed algorithm, presented in Section~\ref{sec:main:algorithm}, constructs e-process-based certificates through a novel sim-to-real betting mechanism inspired in part by prior work by Chen et al.~\cite{chen2026betting}, who use a scalable bank of simulators to select bets and form a bet-weighted \emph{point estimator} that can improve upon the Monte Carlo estimator in~\eqref{eq:mc}. Our work makes three significant advances. 

(i) First, it integrates simulator-guided betting with e-process hypothesis testing to produce a confidence sequence rather than only a point estimate. 

(ii) Second, the coverage result in Section~\ref{sec:main:coverage} establishes that this confidence sequence is simultaneously valid over time whenever the stakes are predictable and safely clipped, regardless of simulator quality. 

(iii) Finally, we analyze when simulator information improves the certificate in Section~\ref{sec:main:bounds}. To the best of our knowledge, we provide the \emph{first} wealth-regret bounds for this class of sim-to-real betting methods (note existing analyses~\cite{kilian2026asymptotically} assume a predictive model that updates on the incoming real observations and do not apply here, Log-loss expert aggregation admits related bounds~\cite{orseau2017soft}, but through a \textsc{Prod}-style update rather than the exponential weights of Algorithm~\ref{alg:sim2real-e}). These bounds provide a precise, method-specific notion of the ``optimal'' simulator(s). More importantly, they quantify how closely any chosen bank of simulators and algorithm configuration approach this optimum and translate that gap into practical guidance for simulator design, bank construction, and hyperparameter selection. 

\textbf{Extended proofs of all theoretical claims are  included in the affiliated code base}.

\subsection{The sim-to-real betting for e-process certificate}\label{sec:main:algorithm}
Algorithm~\ref{alg:sim2real-e} details the proposed method with a bank of $k$ simulators whose performance moments $\big\{\big(\mu_k, \sigma_k^2\big)\big\}_{k=1, \ldots, K}$ may be available analytically or estimated accurately from simulation rollouts. Recall from Section~\ref{sec:intro:probci:hypothe} that an e-process can be interpreted as a sequential betting process. The main idea of Algorithm~\ref{alg:sim2real-e} is to use a scalable bank of simulators to determine how to bet, while using only real-world observations as statistical evidence. Before real-world testing begins, each simulator provides its own estimate of the system's performance. More precisely, the $k$-th simulator induces an estimated distribution of the scalar performance measure defined in Section~\ref{sec:intro}. The proposed algorithm represents this estimate through its first two moments: the simulated mean $\mu_k$, which describes the expected performance, and the simulated variance $\sigma_k^2$, which describes its variability.
\begin{algorithm}[h]
    \begin{algorithmic}[1]
\State {\bf Given:} $P$, $\big\{\big(\mu_k, \sigma_k^2\big)\big\}_{k=1, \ldots, K}$, $\psi: \mathcal{X} \to \M$, $T \in \mathbb{Z}$, $\eta>0$, $\alpha \in (0,1)$, $\kappa\in(0,1]$
\State {\bf Initialize:} $L_0^k=0, \forall k \in \mathbb{Z}_K$, $M_0(\mu)=1, \forall \mu \in \M$
\State {\bf For} $t = 1, 2, \ldots, T$:
\State \ \ \ \ Compute
\small{
\begin{equation}\label{eq:sim2real-e-weights}
    \pi_t^k\!=\!\frac{\exp\big(\eta L_{t-1}^k\big)}{\sum_{j=1}^K\!\exp\big(\eta L_{t-1}^j\big)},\!m_t\!=\!\sum_{k=1}^K\!\pi_t^k \mu_k,\!v_t\!=\!\sum_{k=1}^K\!\pi_t^k\sigma_k^2
\end{equation}
}
\State \ \ \ \ Choose the stake
\begin{equation}\label{eq:sim2real-e-stake}
\lambda_t(\mu)\!=\!\text{clip}_{\delta}\big\{ \kappa(m_t-\mu) / \big[v_t + (m_t-\mu)^2\big] \big\}, \forall \mu \in \M
\end{equation}
\State \ \ \ \ Sample $x_t \sim P$ i.i.d., $y_t = \psi(x_t)$
\State \ \ \ \ Update wealth
\begin{equation}\label{eq:sim2real-e-evidence}
    M_t(\mu) = M_{t-1}(\mu)[1 + \lambda_t(\mu)(y_t-\mu)], \forall \mu \in \M
\end{equation}
\State \ \ \ \ Update score
\begin{equation}\label{eq:sim2real-e-score}
    L_t^k = L_{t-1}^k+\log{\mathcal{N}(y_t;\mu_k, \sigma_k^2)}, \forall k \in \mathbb{Z}_K
\end{equation}
\State \ \ \ \ Update certificate $\C_t$ per Remark~\ref{rmk:e-certificate}
\State {\bf Output:} $\{\C_t\}_{t=1,\ldots,T}$
\end{algorithmic}
    \caption{Sim-to-Real Betting for E-Process Certificate}
    \label{alg:sim2real-e}
\end{algorithm}

As real-world observations arrive at each round $t$, they serve two related purposes. First, they update the trust assigned to each simulator according to how well its predicted performance distribution explains the observations (line~4). Second, they provide the only observations used to update the e-process wealth (line~5-7). At each round, before the next real-world observation is revealed, the algorithm combines the simulator means and variances according to their current trust weights. The resulting mean and variance are then used to determine the direction and magnitude of a predictable bet against each candidate value $\mu\in[0,1]$.

Algorithm~\ref{alg:sim2real-e} returns $\{\C_t\}_{t=1,\ldots,T}$, a sequence of certificates. Their quality can be evaluated along two distinct dimensions. The first is \emph{\textbf{coverage}}: whether the sequence is statistically valid, i.e., containing the real-world mean $\mu^*$ at every round with the prescribed probability. The second is \emph{\textbf{tightness}}: subject to maintaining coverage, whether the resulting certificates are sufficiently narrow, as measured by their widths. For each property, one other important question is whether, and in what way, it depends on the design of the simulator bank and the hyperparameters of Algorithm~\ref{alg:sim2real-e}.

\subsection{The coverage guarantee}\label{sec:main:coverage}
This section studies the ``coverage'' property as dictated by Proposition~\ref{prop:coverage}.
\begin{proposition}\label{prop:coverage}[Anytime valid coverage]
    The confidence sequence returned by Algorithm~\ref{alg:sim2real-e} satisfies
    \begin{equation}\label{eq:coverage}
        \mathbb{P}(\mu^* \in \C_t \forall t \in \mathbb{Z}_T) \geq 1-\alpha.
    \end{equation}
\end{proposition}
That is, Algorithm~\ref{alg:sim2real-e} guarantees the simultaneous coverage of $\{\C_t\}_{t=1,\ldots,T}$, regardless of the accuracy of the simulator bank. This provides an important practical convenience: an experimenter may continuously inspect the certificates and select a reporting or stopping round based on the observed data without invalidating the coverage guarantee. Moreover, the ``quality'' of the simulators does not affect coverage. Even an inaccurate simulator bank produces a valid confidence sequence, although the resulting certificates may not be as tight as desired, which leads to our next topic.

\subsection{The tightness guarantee}\label{sec:main:bounds}
To address the second question on \emph{tightness} raised in the end of Section~\ref{sec:main:algorithm}, recall the standard inversion of e-process tests from Remark~\ref{rmk:e-certificate}, for each candidate $\mu\in[0,1]$, the e-process accumulates evidence against that candidate, which is rejected once its wealth reaches $1/\alpha$. The certificate is formed by the candidates that have not been rejected. Consequently, faster wealth growth against false candidates generally leads to their earlier exclusion and hence tighter certificates. \textbf{The study of the \emph{tightness} of $\C$ hence gets transferred to the study of the \emph{wealth} accumulation.}

\subsubsection{Defining the wealth regret}\label{sec:main:bounds:regret}

Recall the ideal Kelly oracle in determining the optimal stake, for $(\mu^{*}, \sigma^2)$, we have the ideal stake as $\lambda^*(\mu)$ and the ideal wealth as $M_t^*(\mu)$ (i.e., the true-moment oracle wealth at step $t$). To facilitate the wealth regret bound analysis, we consider the ideal moments $(\mu^{*}, \sigma^2)$ as an external benchmark of Algorithm~\ref{alg:sim2real-e}.

Recall the e-process discussed in~\ref{para:evalue}, for a candidate $\mu\in\mathcal{M}$ after $t$ observed measures from the real distribution $P$, Algorithm~\ref{alg:sim2real-e} has wealth $M_t(\mu)$, while ideal Kelly has $M_t^*(\mu)$ on the same observations. We define the \emph{wealth regret} as the logarithm of their ratio, i.e.,
\begin{equation}\label{eq:wealth-regret-base}
    \mathcal{R}_t(\mu) = \log M_t^*(\mu) - \log M_t(\mu).
\end{equation}
Its positive part counts only cases where Algorithm~\ref{alg:sim2real-e} has less wealth as
\begin{equation}
    \mathcal{R}_t^+(\mu) = \max{\{ \mathcal{R}_t(\mu), 0 \}} \text{ and } \overline{\mathcal{R}}_t = \sup_{\mu\in\mathcal{M}} \mathcal{R}_t^+(\mu).
\end{equation}
For example, if ideal Kelly's wealth is twice Algorithm~\ref{alg:sim2real-e}'s wealth, the regret is $\log2$, and if Algorithm~\ref{alg:sim2real-e}'s wealth is twice ideal Kelly's wealth, the regret is $-\log2$, and the positive regret $\mathcal{R}_t^+(\mu)$ further sets it to zero. $\overline{\mathcal{R}}_t$ then selects the candidate with the largest positive regret. Dividing by the round number $t$, i.e., $\overline{\mathcal{R}}_t/t$, normalizes the total log-wealth difference by the number of real-world observations. For example, if ideal Kelly has twice the wealth at the worst candidate after $t=100$ observations, then $\overline{\mathcal{R}}_{100}/100=\log(2)/100\approx0.00693$, which is the average rate at which the log-wealth difference accumulates over the 100 observations.

In this paper, a \emph{wealth-regret bound} ultimately means a bound on the expected average worst-candidate positive regret, i.e.,
\begin{equation}\label{eq:exp-wealth-regret}
\mathbb{E}\left[\overline{\mathcal{R}}_t\right]/t.    
\end{equation}
Here, $\overline{\mathcal{R}}_t$ selects the candidate mean for which Algorithm~\ref{alg:sim2real-e} falls furthest behind ideal Kelly in log wealth, the expectation averages this quantity over possible real-world observation sequences (from $P$), and division by $t$ normalizes the accumulated log-wealth difference by the number of real-world observations. 

In the remainder of this section, we derive two complementary upper bounds on~\eqref{eq:exp-wealth-regret}. Each bound provides a different perspective on what constitutes an effective simulator bank and design principles of Algorithm~\ref{alg:sim2real-e}.

\begin{figure}[b]
    \centering
    \includegraphics[width=\linewidth]{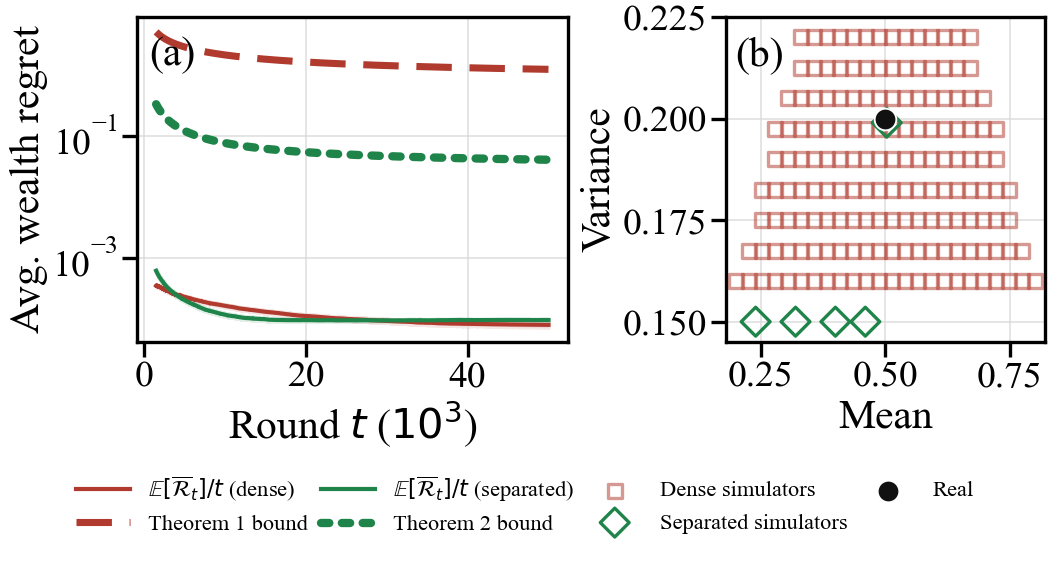}
    \vspace{-5mm}
    \caption{Illustration of Theorem~\ref{thm:general} and \ref{thm:separation-regret} on the same synthetic real distribution: (a) The expected average wealth regret, $\mathbb{E}[\overline{\mathcal{R}}_t]/t$, for dense and separated simulators, together with the corresponding theorem bounds. (b) The mean and variance of the real distribution and each simulator. The two simulator sets yield similar regret while satisfying the different structures illustrated by the two theorems.}
    \label{fig:theory-demo}
\end{figure}

\subsubsection{The expected wealth-regret bound in general}
The first bound requires no additional structural assumptions on the simulator bank beyond those stated in Algorithm~\ref{alg:sim2real-e}. We begin with a lemma that connects moment-estimation error to wealth regret.

\begin{lemma}\label{lma:moment-to-wealth-stability}[Moment-to-wealth stability]
    Consider Algorithm~\ref{alg:sim2real-e}, let $v_{\min}=\min \big\{\sigma^2, \sigma_1^2, \ldots, \sigma_K^2\big\}$, $a_m := \frac{\kappa}{\delta v_{\min}}$, $a_v:=\frac{9\kappa}{16\sqrt{3}\delta v_{\min}^{3/2}}$, $c_{\text{stab}}:=\sqrt{a_m^2 + a_v^2}$. For all $t \in \mathbb{Z}_T$ and for all $\mu \in \M$,
    \begin{subequations}
        \begin{align}
            \mathcal{R}_t^+(\mu) \leq \sum_{\tau=1}^t \big(a_m |m_{\tau} - \mu^*| + a_v \big|v_{\tau} - \sigma^2\big|\big),\label{eq:moment-to-wealth-stability-1} \\
            \overline{\mathcal{R}}_t \leq c_{\text{stab}} \sum_{\tau=1}^t \norm{\big(m_{\tau}-\mu^*, v_{\tau}-\sigma^2\big)}.\label{eq:moment-to-wealth-stability-2}
        \end{align}
    \end{subequations}
\end{lemma}
The constants $a_m$ and $a_v$ measure sensitivity to mean error and variance error, and $c_{\text{stab}}$ further combines the two. To make the comparison, fix a candidate mean $\mu$ and evaluate Algorithm~\ref{alg:sim2real-e} and the ideal-Kelly benchmark on the same sequence of real observations. At round $t$, Algorithm~\ref{alg:sim2real-e} inserts its weighted pair $(m_t,v_t)$ into the stake formula, whereas ideal Kelly inserts the true pair $(\mu^*,\sigma^2)$; all other inputs are the same. The term~\eqref{eq:moment-to-wealth-stability-1} sums the effects of these mean and variance errors over the first $t$ rounds and bounds the positive wealth regret for the fixed candidate $\mu$. Because the same bound holds for every candidate mean, the term~\eqref{eq:moment-to-wealth-stability-2} also bounds the largest positive wealth regret over all candidates.

Lemma~\ref{lma:moment-to-wealth-stability} reduces the wealth-regret analysis to a more concrete question: how closely do the trust-weighted mean and variance produced by Algorithm~\ref{alg:sim2real-e} track the unknown real mean and variance? It does not, by itself, establish that these moment errors are small. The following theorem answers this question for a general simulator bank. It compares the trust-weighted simulator score with that of the best fixed simulator in the bank, converts the resulting score difference into moment error, and then applies Lemma~\ref{lma:moment-to-wealth-stability}.

\begin{theorem}\label{thm:general}[Expected wealth regret]
    Consider Algorithm~\ref{alg:sim2real-e} and Lemma~\ref{lma:moment-to-wealth-stability}. Let $v_{\max}=\max \big\{\sigma^2, \sigma_1^2, \ldots, \sigma_K^2\big\}$, $b_{\text{score}}\!=\!\frac{1}{2}\log\frac{v_{\max}}{v_{\min}} + \frac{1}{2v_{\min}}$, $c_{\text{mom}}=2v_{\max}+4v_{\max}^2$,   $\epsilon_K:=\min_{k\in\mathbb{Z}_K}\frac{1}{2}\Big[ \log\Big(\frac{\sigma_k^2}{\sigma^2}\Big) + \frac{\sigma^2 + (\mu^*-\mu_k)^2)}{\sigma_k^2} - 1 \Big]$, we have~\eqref{eq:exp-wealth-regret} obeys
    \begin{equation}\label{eq:general-regret}
\frac{\mathbb{E}\!\left[\overline{\mathcal{R}}_t\right]}{t} \leq c_{\text{stab}} \sqrt{c_{\text{mom}}\Bigg( \epsilon_K + \frac{\log K}{\eta t} + \frac{\eta b_{\text{score}}^2}{8}  \Bigg) }.
    \end{equation}
    The expectation is over the i.i.d. real stream from $P$. The supremum is over candidates at one fixed round $t$.
\end{theorem}

Ignoring the multiplicative sensitivity factor $c_{\mathrm{stab}}\sqrt{c_{\mathrm{mom}}}$, the R.H.S. is separated to three primary terms. The first, $\epsilon_K$, is the expected Gaussian-loss difference between the best fixed simulator (in the bank of simulators) and the true moment pair (it is zero if the true moment is a part of the bank of simulators). If this best simulator were known beforehand, $\epsilon_K$ would be the only term inside the parentheses. In practice, Algorithm~\ref{alg:sim2real-e} does not know which simulator is the best and initially assigns trust across all $K$ simulators. The term $\log K/(\eta t)$ is the resulting initial identification cost, which decreases as more real observations are collected. The term $\eta b_{\mathrm{score}}^2/8$ accounts for fluctuations in the simulator scores while the trust weights are being updated. A larger learning rate $\eta$ reacts more aggressively to these fluctuations and therefore increases this term.

\begin{remark}\label{rmk:general-regret}[``Good'' simulators under Theorem~\ref{thm:general}] Theorem~\ref{thm:general} identifies two main factors governing the effectiveness of Algorithm~\ref{alg:sim2real-e}. First, the simulator bank should contain at least one simulator whose mean and variance are close to the corresponding real-world moments measured by $\epsilon_K$ in~\eqref{eq:general-regret}. A \emph{denser} bank thus can reduce $\epsilon_K$ by increasing the chance of including such a moment-accurate simulator, while the direct cost of enlarging the bank grows only through $\log K$. Second, the learning rate $\eta$ should balance the two related terms in~\eqref{eq:general-regret}. For a fixed round $t$, such a balance gives $\eta=\sqrt{\frac{8\log K}{t\,b_{\mathrm{score}}^2}}$. \textbf{Thus, a larger bank generally requires a moderately larger learning rate}, whereas a larger $T$ or greater variation in simulator scores calls for a smaller learning rate.
\end{remark}

\begin{figure*}
\begin{minipage}[c]{0.6\textwidth}
    \centering

\begin{subfigure}[c]{0.9\textwidth}
  \centering
  \includegraphics[width=\linewidth]{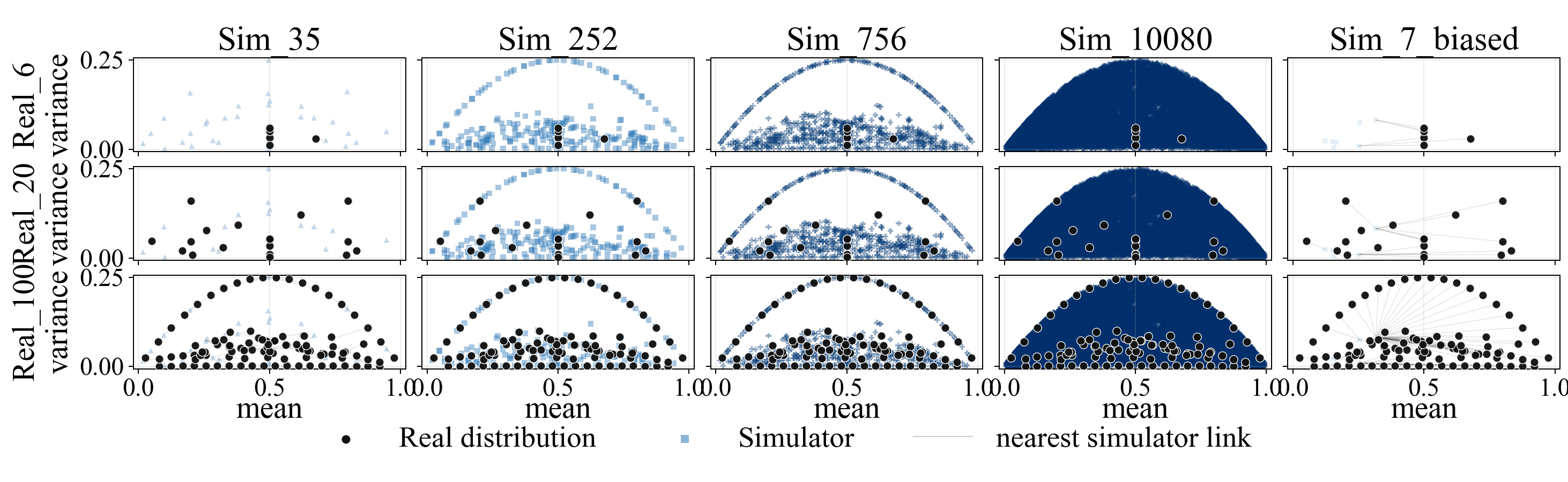}
  \caption{\textbf{C1} - synthetic examples}
  \label{fig:exp:synthetic}
\end{subfigure}

\begin{subfigure}[c]{0.9\textwidth}
  \centering
  \includegraphics[width=\linewidth,trim=20 70 20 65, clip]{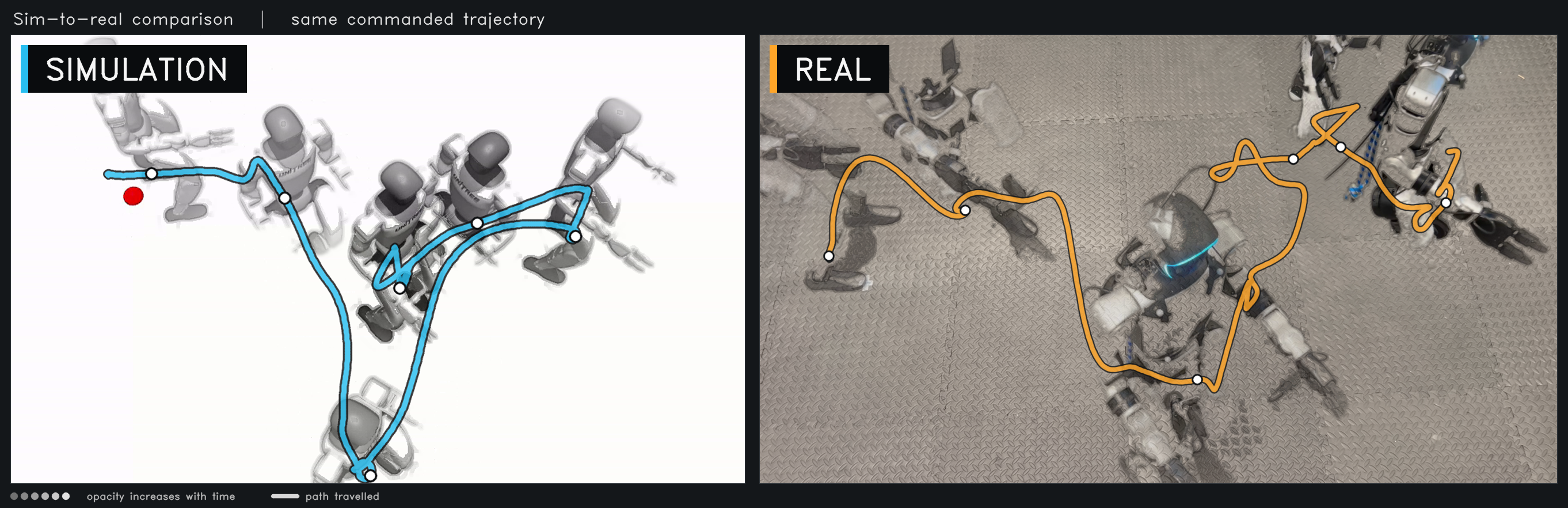}
  \caption{\textbf{C2} - command tracking}
  \label{fig:exp:gr00t}
\end{subfigure}
\end{minipage}
\hfill
\begin{subfigure}[c]{0.19\textwidth}
  \centering
  \includegraphics[width=\linewidth,trim=0 0 0 30, clip]{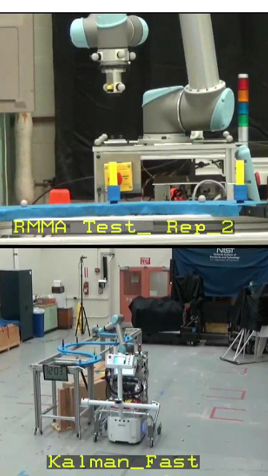}
  \caption{\textbf{C3} - peg-in-hole}
  \label{fig:exp:nist}
\end{subfigure}\hfill
\begin{subfigure}[c]{0.19\textwidth}
  \centering
  \includegraphics[width=\linewidth,trim=0 0 200 95, clip]{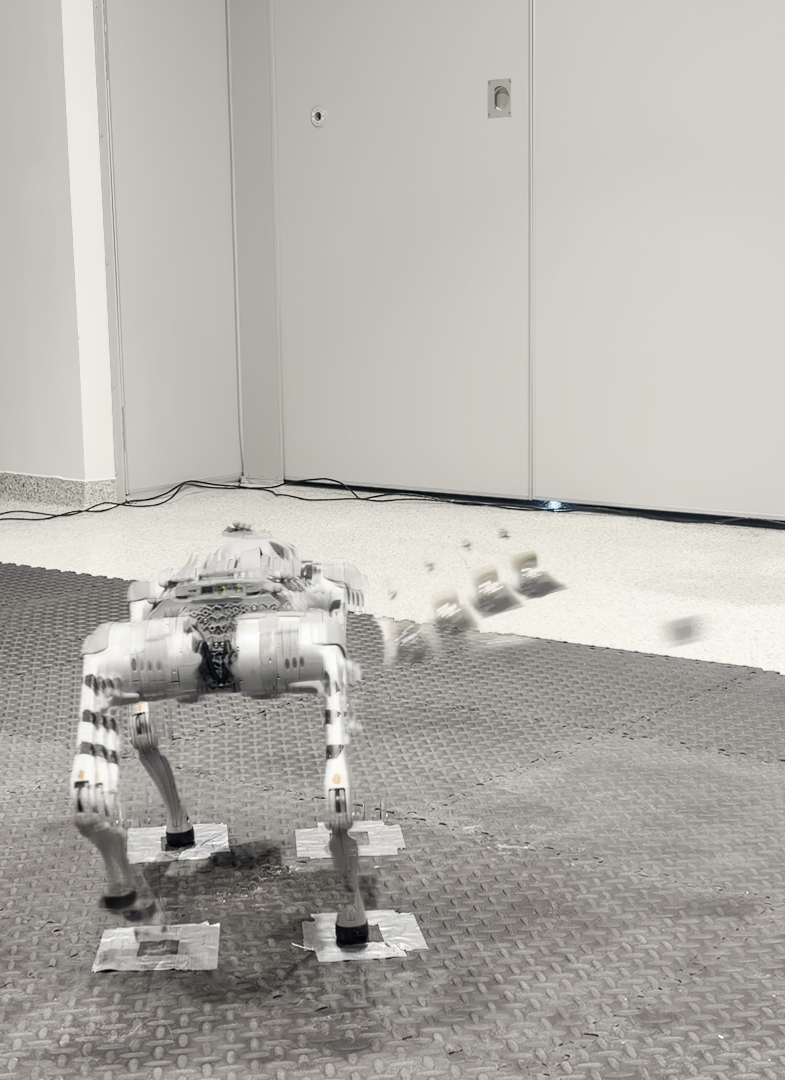}
  \caption{\textbf{C3} - pushover stability}
  \label{fig:exp:astm}
\end{subfigure}
\caption{An overview of the three experimental categories and four testing tasks is provided below. For the \emph{synthetic evaluation}, Fig.~\ref{fig:exp:synthetic} illustrates five simulator banks relative to three sets of target, or ``real'' distributions. Both the simulators and the target distributions are synthetic and are constructed from parameterized distribution families, including Bernoulli, beta, truncated-normal, bimodal-mixture, uniform-spike, and Gaussian-mixture distributions (adapted and extended from~\cite{weng26repeatability, chen2026betting}). Each individual distribution is represented as a point in the mean-variance plane, with different markers distinguishing the target distributions from the simulators. Each simulator-bank name begins with \texttt{Sim} followed by the number of distinct parameterized distributions contained in that bank. For example, \texttt{Sim\_756} contains 756 simulator distributions. An additional suffix describes a special bank construction when needed, as in \texttt{Sim\_7\_biased}. A similar naming strategy applies to the group of \texttt{Real} distributions as well.
Fig.~\ref{fig:exp:nist} presents the open-source NIST Continuous Mobile Manipulator Performance Measurement Dataset~\cite{nist24continuous}. The dataset records a wheeled mobile manipulator performing a peg-in-hole task according to standardized testing procedures.
Fig.~\ref{fig:exp:astm} presents a test conducted under ASTM Work Item WK86916~\cite{astm23test} developed by ASTM Subcommittee F45.06 to standardize the evaluation of legged robots under push-over disturbances. In this test, a Unitree Go2 robot operating with its manufacturer-provided control modules is subjected to controlled impacts delivered by a swinging-pendulum impactor. 
Fig.~\ref{fig:exp:gr00t} showcases (a simulator and a real views of) a command tracking accuracy testing task of a Unitree G1 humanoid running the GR00T~\cite{gr00tgithub} locomotion controller under joystick commands. 
}
\label{fig:exp}
\end{figure*}

\subsubsection{The expected wealth-regret bound under simulator separation}
Different from the general bound in Theorem~\ref{thm:general}, the second bound considers a particular distribution of simulator quality within the bank, where one reference simulator stands out from its competitors. By paying the price of this additional separation assumption, we can obtain a tighter wealth-regret bound. The following lemma formalizes this condition.

\begin{lemma}\label{lma:trust-concentration}[Simulator separation]
    Consider Algorithm~\ref{alg:sim2real-e} with $K \geq 2$ and $\eta>0$. Let the bank member $k' \in \mathbb{Z}_K$ as the reference. Define the relative trust and its one-round multiplier as $r_t^k = \pi_t^k/\pi_t^{k'}$ and $G_t^k = r_{t+1}^k/r_t^k$, respectively for all $k \in \mathbb{Z}_K \setminus \{k'\}$. If for some $\rho$ satisfying $\mathbb{E}\big[G_t^k\big]\leq \rho \leq 1, \forall k \in \mathbb{Z}_K \setminus \{k'\}, \forall t \in \mathbb{Z}_{T-1}$, every competitor's expected relative trust, and hence the expected total trust outside the reference, obey
    \begin{equation}\label{eq:trust-concentration}
        \mathbb{E}\big[r_t^k\big]\leq \rho^{t-1}, \mathbb{E}\big[1-\pi_t^{k'}\big]\leq \min \big\{ 1, (K-1)\rho^{t-1} \big\}.
    \end{equation}
\end{lemma}

The ratio $r_t^k$ compares competitor simulator $k$'s trust with any chosen reference's trust. The multiplier $G_t^k$ is the factor by which observing $y_t$ changes this ratio. The assumption on $\rho$ states that such factor is at most $\rho$ on average. A larger $\rho$ means that, on average, trust shifts toward the reference more slowly. Lemma~\ref{lma:trust-concentration} consequently shows that the expected trust outside the reference decreases geometrically. The following theorem combines this result with Lemma~\ref{lma:moment-to-wealth-stability} to obtain the second wealth-regret bound.

\begin{theorem}\label{thm:separation-regret}[Expected wealth regret with simulator separation]
    Consider Algorithm~\ref{alg:sim2real-e} with $K \geq 2$. Fix the reference simulator $k'$ satisfying Lemma~\ref{lma:trust-concentration} for some $\rho < 1$. Let $h_m:=|\mu_{k'}-\mu^*|, h_v:=\big|\sigma_{k'}^2-\sigma^2\big|$, $D_m:= \max_{k\in\mathbb{Z}_K}|\mu_k-\mu_{k'}|, D_v:= \max_{k\in\mathbb{Z}_K}\big|\sigma^2_k-\sigma^2_{k'}\big|$, we have
    \begin{equation}\label{eq:S_t-bound}
        S_t\!:=\!\sum_{\tau=1}^t\!\min\!\big\{\! 1,\!(\!K\!-\!1)\rho^{\tau\!-\!1} \!\big\}\!\leq\!1\!+\!\Bigg[\! \frac{\log(K\!-\!1)}{-\log \rho} \!\Bigg]\!+\!\frac{1}{1\!-\!\rho}.
    \end{equation}
    The same expected average regret as in Theorem~\ref{thm:general} obeys
    \begin{equation}\label{eq:seperation-regret}
\frac{\mathbb{E}\!\left[\overline{\mathcal{R}}_t\right]}{t} \leq a_m\Bigg( h_m + D_m \frac{S_t}{t} \Bigg) + a_v\Bigg( h_v + D_v \frac{S_t}{t} \Bigg).
    \end{equation}
\end{theorem}

Recall that Remark~\ref{rmk:general-regret} characterizes a good simulator bank under Theorem~\ref{thm:general} primarily by whether it contains at least one moment-accurate simulator, which needs not stand out from the rest of the bank. Theorem~\ref{thm:separation-regret} provides a different characterization summarized by the following remark.

\begin{remark}\label{rmk:separation-regret}[``Good'' simulators under Theorem~\ref{thm:separation-regret} (and Lemma~\ref{lma:trust-concentration})]
Theorem~\ref{thm:separation-regret}, together with Lemma~\ref{lma:trust-concentration}, identifies two distinct properties of a good simulator bank. 
First, the bank must contain an accurate reference simulator whose mean and variance are close to the corresponding real-world moments. These persistent moment errors are measured by $h_m$ and $h_v$. Second, the reference must be distinguishable from every competing simulator: the expected one-round multiplier of each competitor-to-reference trust ratio must be bounded by a common $\rho<1$. Under this separation condition, trust concentrates on the reference simulator over time. A smaller $\rho$ implies faster concentration and therefore a more rapidly decreasing simulator-selection term $S_t/t$. The quantities $D_m$ and $D_v$ determine how strongly the remaining trust on competing simulators affects the aggregated moments. If the reference has exactly the real-world mean and variance, the persistent terms vanish, and the resulting expected average wealth-regret bound decreases at rate $O(1/t)$. The price of this tighter rate is the stronger separation assumption, which may fail when several simulators receive nearly identical scores.
\end{remark}

We conclude this section with a toy example demonstrating the two remarks and the related theoretical results presented to this point. 

Fig.~\ref{fig:theory-demo} applies Algorithm~\ref{alg:sim2real-e} to the same synthetic real distribution, with mean $\mu^*=0.5$ and variance $\sigma^2=0.2$, using two differently structured simulator banks. Panel~(b) visualizes their moment pairs. The \emph{dense bank} of simulators covers a broad region around the real moment pair (without overlap), illustrating Remark~\ref{rmk:general-regret}. However, its many similarly accurate simulators prevent any single simulator from clearly standing out (which fails the separation condition of Lemma~\ref{lma:trust-concentration}). In contrast, the separated bank contains one simulator that is favored over its competitors under real-world observations, illustrating the additional separation condition in Remark~\ref{rmk:separation-regret}. 

Panel~(a) shows the expected average wealth regret $\mathbb{E}[\overline{\mathcal{R}}_t]/t$ for Algorithm~\ref{alg:sim2real-e} under the two simulator banks, together with their corresponding theoretical bounds. Theorem~\ref{thm:general} is applied to the dense bank without requiring a distinguished simulator, whereas Theorem~\ref{thm:separation-regret} is applied to the separated bank and obtains a tighter bound by exploiting its stronger trust-contraction assumption. The two solid curves have similar empirical performance, emphasizing that the two results provide complementary explanations of the same algorithm under different simulator-bank structures.

\begin{figure*}
  \centering
  \begin{minipage}[c]{0.5\textwidth}
    \centering
    \begin{subfigure}[b]{\linewidth}
      \centering
      \includegraphics[width=\linewidth]{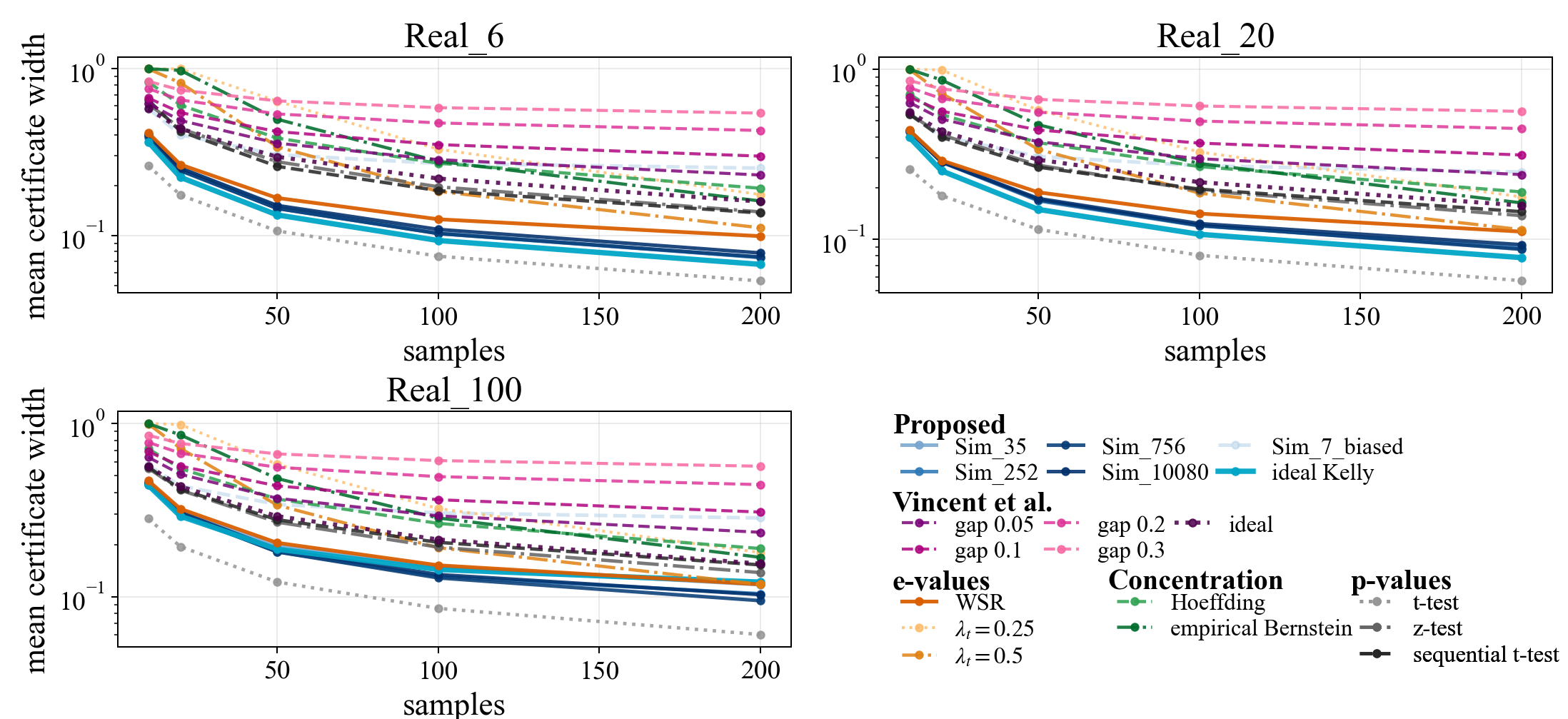}
      \caption{The evolution of the average certificate width for each method. }\label{fig:syn:width}
    \end{subfigure}
    
    \begin{subfigure}[b]{\linewidth}
      \centering
      \includegraphics[width=\linewidth]{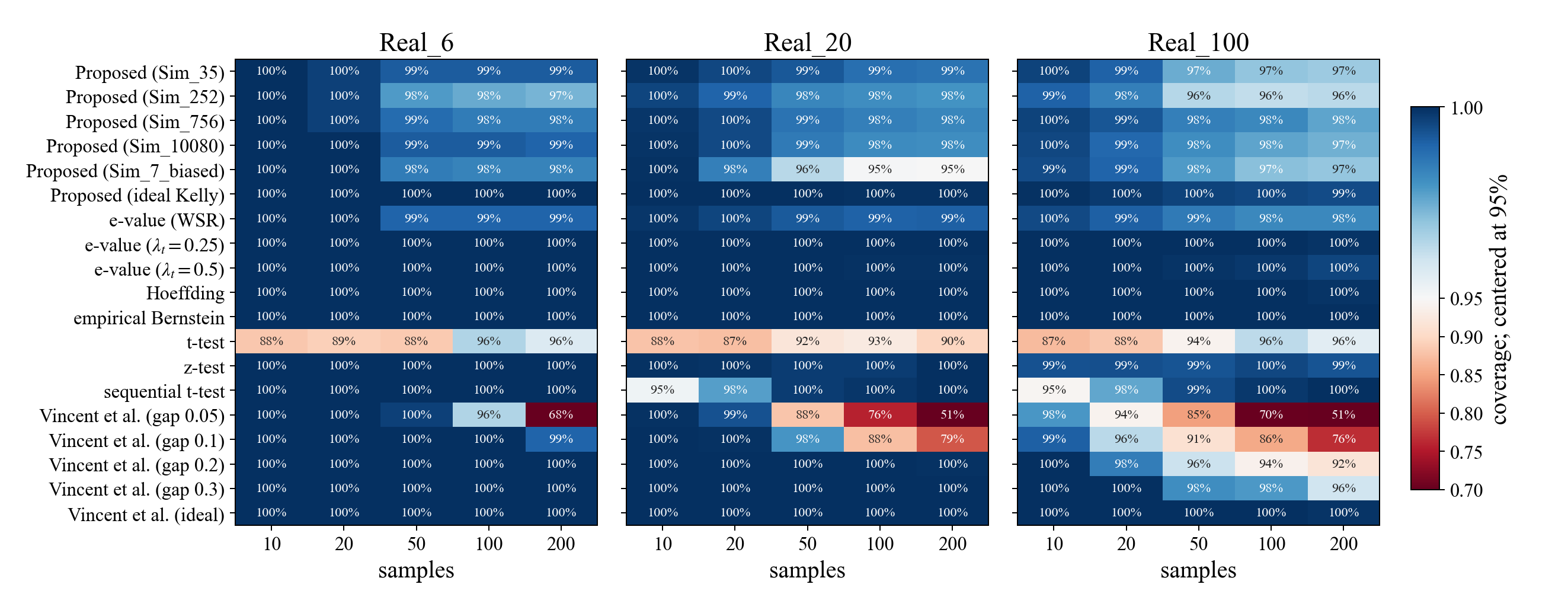}
      \caption{Empirical coverage of various methods.}
      \label{fig:syn:coverage}
    \end{subfigure}
  \end{minipage}
  \hfill
  \begin{subfigure}[c]{0.47\textwidth}
    \centering
    \includegraphics[width=0.9\linewidth,trim=0 0 0 30, clip]{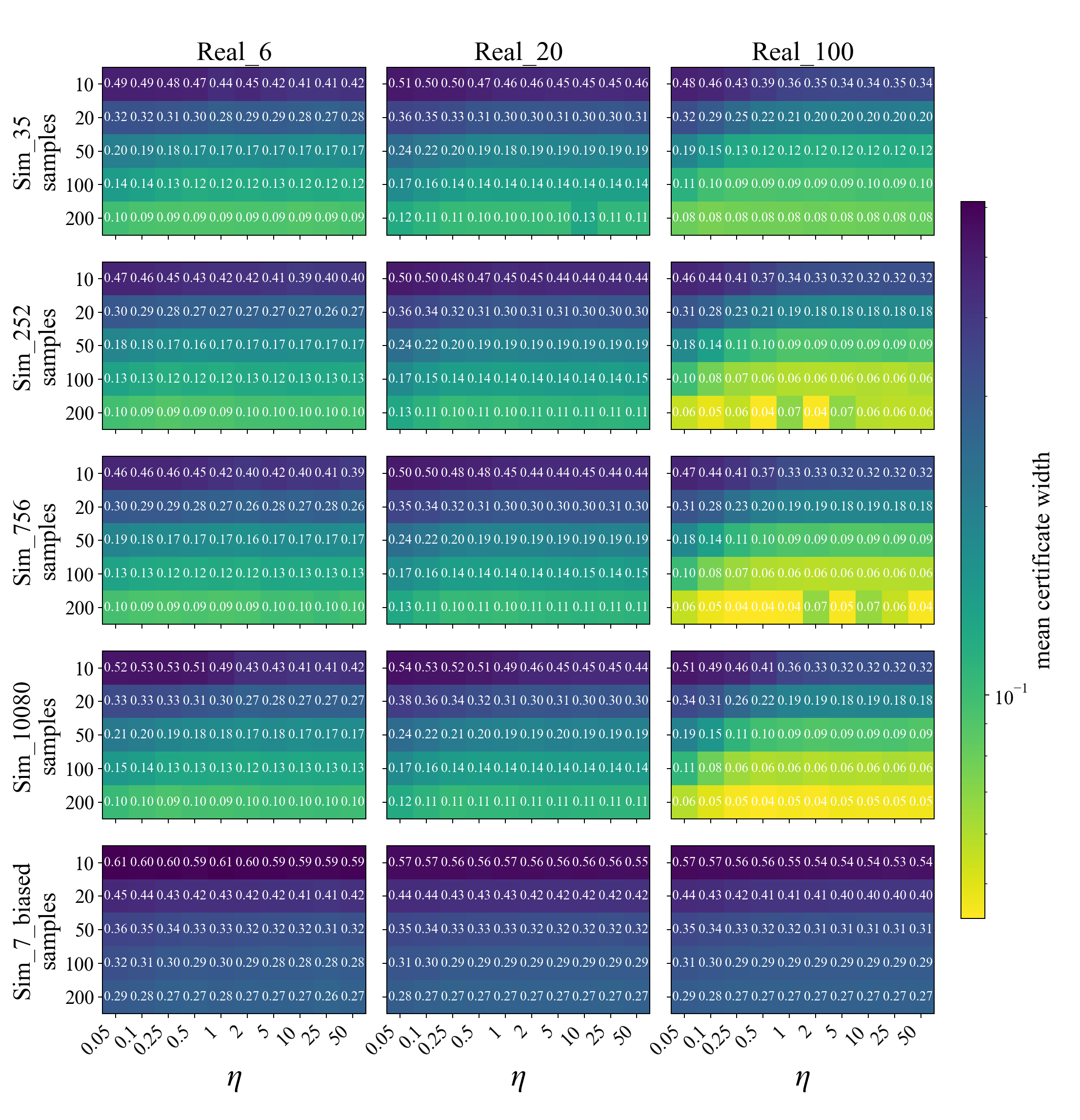}
    \caption{An empirical ablation study of $\eta$.}\label{fig:syn:eta}
  \end{subfigure}
  \caption{Empirical results of \textbf{C1}. At every reported sample size ($t$), Fig.~\ref{fig:syn:width}'s certificate widths are averaged over all distributions in the corresponding \texttt{Real} group and over 100 independently seeded trials for each distribution. Fig.~\ref{fig:syn:coverage}'s empirical coverage is shown as the observed probability for the obtained $\C_t$ covering the analytical $\mu^*$ (note the theoretical coverage is 95\% given $\alpha=0.05$.).
  Fig.~\ref{fig:syn:eta}'s width is calculated similar to that in Fig.~\ref{fig:syn:width}.}
  \label{fig:synthetic}
\end{figure*}

Bringing these insights to practical robot testing, Theorem~\ref{thm:general} and its dense-bank design principle are particularly useful when little prior information is available about which simulator, if any, will accurately represent the robot under evaluation. One example is evaluating an imitation-learning policy trained directly from real-world demonstrations when no well-matched simulator has been developed~\cite{chen2026betting,snyder26beyondbinary}. In contrast, Theorem~\ref{thm:separation-regret} is relevant when the testing operator or stakeholder already has a simulator design that approximately mirrors the real system, although its exact physical parameters remain uncertain. In this setting, domain randomization~\cite{peng18rand} can be applied within a local neighborhood of the nominal simulator to construct a focused bank of plausible variants. These conclusions resonate with the earlier sim-to-real betting and robot-testing results of Chen et al.~\cite{chen2026betting}, particularly Remark~4 and the associated experiments.

\begin{figure}
    \centering
\includegraphics[width=\linewidth,trim=0 0 0 0, clip]{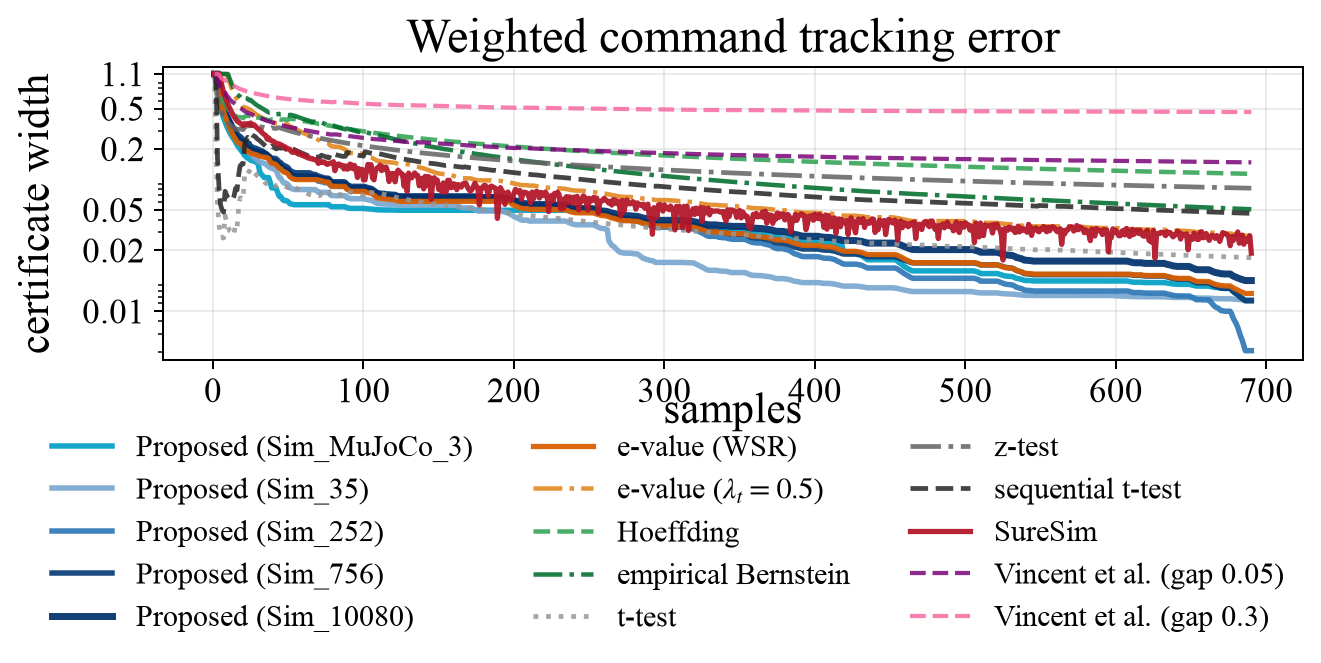}
    \caption{The evolution of the average certificate width for each method for the \textbf{C2} task. Given the command $x^c = [v_x^c, v_y^c, \omega^c]$ (i.e., desired linear velocities and yaw rate) and the measured robot state $x^r=[v_x^r, v_y^r, \omega^r]$, the weighted command tracking error takes the form of $0.25\sqrt{(v_x^c-v_x^r)^2 + (v_y^c-v_y^r)^2} + 0.5|(\omega^c-\omega^r|$.}
    \label{fig:gr00t}
\end{figure}

\begin{figure*}
\centering
\begin{subfigure}[c]{\textwidth}
  \centering
  \includegraphics[width=\linewidth,trim=0 75 0 0, clip]{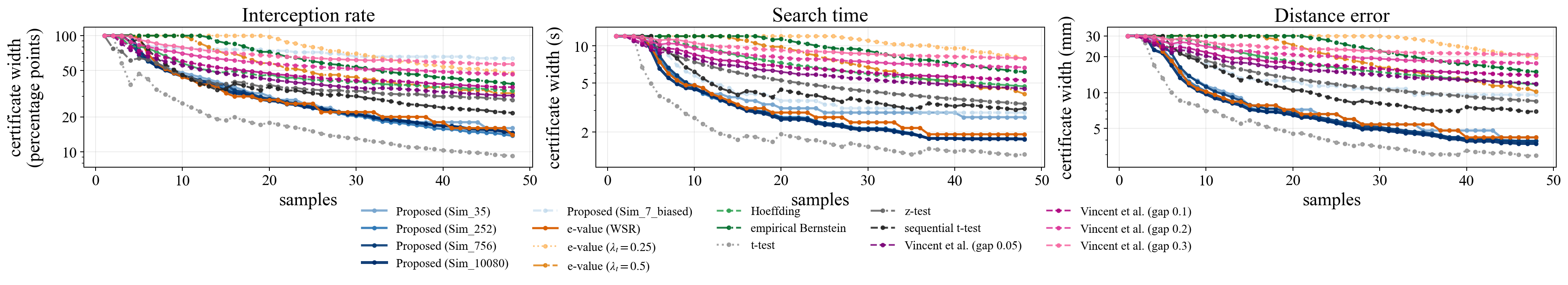}
  \caption{The evolution of the average certificate width for various performance measures in the \textbf{C3}-peg-in-hole task replaying the NIST dataset~\cite{nist24continuous}. \texttt{Interception rate} measures the fraction of trials in which the end effector successfully intercepts a reflector target serving as the non-contact proxy for peg-in-hole insertion. The \texttt{Search time} measures the expected time to complete one interception. The \texttt{Distance error} is a continuous measure complementing the binary \texttt{Interception rate} with direct measure of the $\ell_2$-norm distance between the manipulator's end effector and the target center.}
  \label{fig:nist-width}
\end{subfigure}
\begin{subfigure}[c]{\textwidth}
  \centering
  \includegraphics[width=\linewidth]{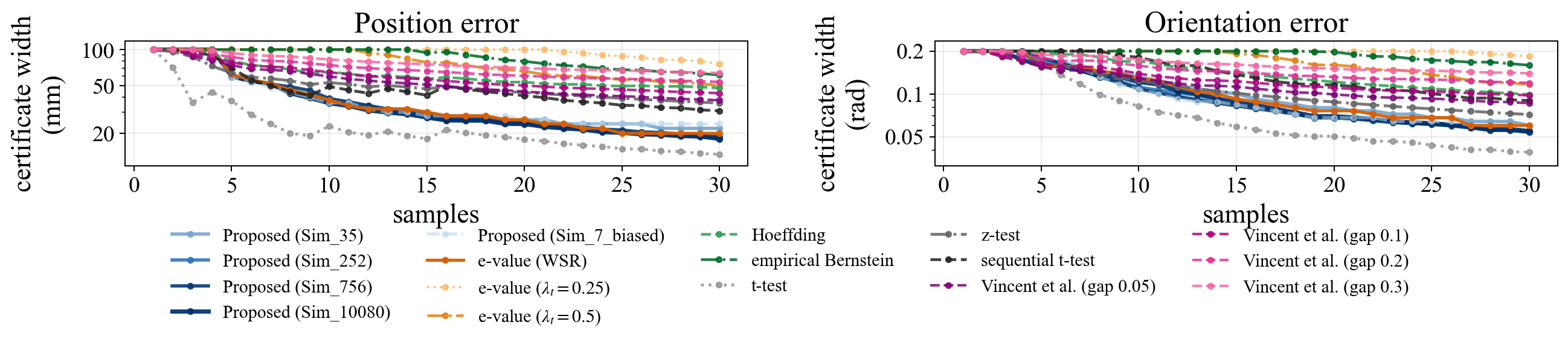}
  \caption{The evolution of the average certificate width for two performance measures in the \textbf{C3} push-over task. \texttt{Position error} measures the $\ell_2$-norm distance and \texttt{Orientation error} the absolute yaw-angle difference between the stabilized poses before and after the impact event.}
  \label{fig:astm-width}
\end{subfigure}
\caption{The evolution of the average certificate width for various methods in two different \textbf{C3} tasks.}
\label{fig:data_replay}
\end{figure*}

\section{Experiment}\label{sec:exp}

To demonstrate the proposed algorithm and empirically illustrate the theoretical insights developed in Section~\ref{sec:main}, our experimental study consists of four groups of tests organized into three categories (see Fig.~\ref{fig:exp}). 

(\textbf{C1}) The \emph{synthetic evaluation} (see Fig.~\ref{fig:exp:synthetic}) uses parameterized distributions to represent both the real-world and simulation data-generating processes. The known true real-world mean and variance allow us to directly evaluate certificate coverage and tightness and to compare the proposed method with the ideal-Kelly oracle. 

(\textbf{C2}) The \emph{command tracking} (see Fig.~\ref{fig:exp:gr00t}) deploys the proposed method to evaluate the command-following accuracy of a Unitree G1 humanoid robot controlled by GR00T~\cite{gr00tgithub}. The simulator bank is constructed in two ways, some directly adapt from \textbf{C1}, and \texttt{Sim\_MuJoCo\_3} reproduces the same i.i.d. command-following operations in MuJoCo~\cite{todorov2012mujoco} through three variants that differ in physical parameters such as joint damping. The variants are configured to exhibit separated performance, creating the trust gap required for Lemma~\ref{lma:trust-concentration} to take effect.

(\textbf{C3}) The \emph{standardized test replay} (see Fig.~\ref{fig:exp:nist} and \ref{fig:exp:astm}) consists of two studies using outcomes from prior robot tests conducted with standardized or near-standardized procedures and apparatus~\cite{astm23test,nist24continuous}. We replay the released or collected outcome sequences and apply the certificate methods to the same observations, enabling comparisons on realistic testing data without rerunning the physical experiments. Because these pre-collected datasets do not include simulator rollouts or directly matched simulators, the sim-to-real methods are supplied with synthetic simulator banks from \textbf{C1}. These banks provide auxiliary simulator information without assuming the availability of a digital twin specifically calibrated to the tested robot. 

The setup of tests are explained and illustrated in Fig.~\ref{fig:exp} with empirical results presented in Fig.~\ref{fig:synthetic}, \ref{fig:gr00t}, and \ref{fig:data_replay}. 

In summary, the proposed Algorithm~\ref{alg:sim2real-e} yields the tightest certificate across nearly all tasks and sample sizes, spanning synthetic distributions to varied performance measures of physical robot systems (including the mobile manipulator, the quadruped, and the humanoid robot). The only exception is the $t$-test, which occasionally produces narrower bounds but, due to its fixed-sample parametric nature, lacks the coverage guarantee, as theoretically supported in~\cite{Ramdas25eprocess}, empirically highlighted in Fig.~\ref{fig:syn:coverage}, and in contrast to Proposition~\ref{prop:coverage} for Algorithm~\ref{alg:sim2real-e}. 

Combine the shared comparisons across Fig.~\ref{fig:syn:width}, \ref{fig:gr00t}, \ref{fig:nist-width}, and \ref{fig:astm-width}, comparing against 9 applicable baselines (excluding SureSim (only applicable for \textbf{C2}), some ``ideal'' variants (only applicable for \textbf{C1}), and $t$- \& $z$-test (lack of coverage guarantee)), the proposed variants of Algorithm~\ref{alg:sim2real-e} achieve a certificate width reduction of $51.6\%\pm16\%$ (even the worst proposed variant, \texttt{Sim\_7\_biased}, reduces the width by $22.9\%\pm60\%$). With extremely limited samples, the proposed methods remain dominate with the width reduction level of $10\%\pm3\%$ for $\leq10$ rounds, $32.26\%\pm8\%$ for $\leq30$ rounds, and $37.3\%\pm12\%$ for $\leq50$ rounds.

The ablation study in Fig.~\ref{fig:syn:eta}, using synthetic distributions, empirically supports the insight from Theorem~\ref{thm:general} and Remark~\ref{rmk:general-regret} that a denser bank generally requires a moderately large $\eta$. On the other hand, when a paired and tunable simulation is available (e.g., with \textbf{C2}), Algorithm~\ref{alg:sim2real-e} with \emph{three} simulators (\texttt{Sim\_MuJoCo\_3}), provided they bear the appropriate separation (in alignment with Theorem~\ref{thm:separation-regret}), achieves overall performance on par with much denser and scaled synthetic banks, and even converges faster in the early stage with limited samples.

\section{Future Work}\label{sec:end}
Future work includes new sim-to-real betting mechanisms, online adaptive simulators and integration into standardized test procedures as a formal asset.

\bibliographystyle{IEEEtran}
\bibliography{output}

\clearpage
\onecolumn
\noindent{\LARGE Supplementary Material: ``Betting for Sim-to-Real Performance Certificate''}

\section{Proofs}\label{apx:prf}

\subsection{Proof for Proposition~\ref{prop:coverage}}
For a fixed candidate $\mu\in\mathcal{M}$ with $\mathcal{F}_{t-1}$ denoting the real observations available before $y_t$ is obtained, the process in Algorithm~\ref{alg:sim2real-e} is a non-negative super-martingale under the candidate null hypothesis $H_0(\mu):\mu^*=\mu$, as $M_0(\mu)=1,\  \mathbb{E}[M_t(\mu) \mid \mathcal{F}_{t-1}] \leq  M_{t-1}(\mu) (t=2,\ldots,T)\ \text{and}\  M_t(\mu)\geq0$. By applying Ville's inequality~\cite{ville1939etude} to the process, the acception and rejection of the candidates is any-time valid, so we have $\mathbb{P}(\forall t \leq T: \mu^* \in \C_t) \geq 1-\alpha$.

\subsection{Proof for Lemma~\ref{lma:moment-to-wealth-stability}}\label{apx:prf:lma:moment-to-wealth-stability}

At every round, Algorithm~\ref{alg:sim2real-e} and the true-moment benchmark observe the same real measure and test the same candidate $\mu$. They thus only differ mean and variance inserted into the stake calculation. The following proof, at a high level, takes three linked steps: (i) moment error changes the stake, (ii) stake error changes the logarithm of the one-round wealth factor, and further (iii) changes the one-round log wealth.

\begin{proof}
    
    \textbf{Step 1: How moment error changes the stake}:
    Fix a candidate $\mu \in \mathcal{M}$ and a round $\tau$. Before the $\text{clip}_{\delta}$ operation, the stake obtained from an inserted mean $m$ and variance $v$ is $\frac{\kappa(m-\mu)}{v+(m-\mu)^2}$. The absolute slope of this expression w.r.t. $m$ (holding the variance fixed) satisfies
    \begin{equation}
        \Bigg| \frac{\kappa\bigl(v-(m-\mu)^2\bigr)}{\bigl(v+(m-u)^2\bigr)^2} \Bigg| \leq \frac{\kappa}{v+(m-\mu)^2} \leq  \frac{\kappa}{v_{\min}}.
    \end{equation}
    Similarly, the absolute slope w.r.t. $v$ is 
    \begin{equation}
        \frac{\kappa|m-\mu|}{\bigl(v+(m-\mu)^2\bigr)^2} \leq \text{$\frac{9\kappa}{16\sqrt{3} v^{3/2}}$} \leq \frac{9\kappa}{16\sqrt{3} v_{\min}^{3/2}}.
    \end{equation}
    The first inequality follows because, as a function of $z=|m-\mu|\geq 0$, the L.H.S. is maximized at $z=\sqrt{v/3}$, equivalently $(m-\mu)^2=v/3$. This explains the numerical constant in $a_v$ with~\eqref{eq:moment-to-wealth-stability-1}.

    \textbf{Step 2: Combine the two coordinate changes}:
    Consider the two raw stakes (without $\text{clip}_{\delta}$), first change the inserted mean from $m_{\tau}$ to $\mu^*$ while keeping $v_{\tau}$ fixed, and then change the variance from $v_{\tau}$ to $\sigma^2$ while keeping $\mu^*$ fixed. The total change cannot exceed the sum of those two changes. Moreover, the $\text{clip}_{\delta}$ operation cannot enlarge this difference. As a result:
    \begin{equation}\label{eq:stake-diff}
        \big|\lambda_{\tau}(\mu) - \lambda_{\tau}^*(\mu)\big| \leq \frac{\kappa}{v_{\min}}|m_{\tau} - \mu^*| + \frac{9\kappa}{16\sqrt{3} v_{\min}^{3/2}} \big|v_{\tau} - \sigma^2\big|.
    \end{equation}
    
    \textbf{Step 3: Stake difference changes the one-round log wealth}: Since both clipped stakes lie in the same interval, so every $\lambda$ between them also lie in that interval. Consequently, safe clipping gives $1+\lambda(y_{\tau}-\mu)>\delta$. Therefore,
    \begin{equation}
        \Bigg| \frac{\partial}{\partial \lambda} \log \bigl[1+\lambda(y_{\tau}-\mu)\bigr] \Bigg| = \frac{y_{\tau}-\mu}{1+\lambda(y_{\tau}-\mu)} \leq \frac{y_{\tau}-\mu}{\delta}.
    \end{equation}
    Combining this fact with the stake bound~\eqref{eq:stake-diff} gives
    \begin{equation}\label{eq:one-round-log-wealth}
        \Bigg| \log \frac{1+\lambda_{\tau}^*(\mu)(y_{\tau}-\mu)}{1+\lambda_{\tau}(\mu)(y_{\tau}-\mu)} \Bigg| \leq \frac{y_{\tau}-\mu}{\delta}\Biggl(\frac{\kappa}{v_{\min}}|m_{\tau} - \mu^*| + \frac{9\kappa}{16\sqrt{3} v_{\min}^{3/2}} \big|v_{\tau} - \sigma^2\big|\Biggr).
    \end{equation}

    \textbf{Step 4: Sum the one-round bounds to obtain wealth regret}:
    Using the product definitions of $M_t^*(\mu)$ and $M_t(\mu)$, followed by~\eqref{eq:one-round-log-wealth}, gives
    \begin{equation}
        \begin{aligned}
            \mathcal{R}_t^+(\mu) &= \max\Bigg\{ \sum_{\tau=1}^t \log \frac{1+\lambda_{\tau}^*(\mu)(y_{\tau}-\mu)}{1+\lambda_{\tau}(\mu)(y_{\tau}-\mu)}, 0 \Bigg\} \\
            &\leq \sum_{\tau=1}^t \Biggl|\log \frac{1+\lambda_{\tau}^*(\mu)(y_{\tau}-\mu)}{1+\lambda_{\tau}(\mu)(y_{\tau}-\mu)}\Biggr| \\
            & \leq \sum_{\tau=1}^t (y_{\tau}-\mu)\bigl(a_m |m_{\tau} - \mu^*| + a_v \big|v_{\tau} - \sigma^2\big|\bigr)
        \end{aligned}
    \end{equation}
    The first inequality bounds the positive part of a sum by the sum of the absolute values. The second inequality applies \eqref{eq:one-round-log-wealth} at every round. This proves~\eqref{eq:moment-to-wealth-stability-1}.

    Finally, consider
    \begin{align*}
      &c_{\mathrm{stab}}^2
      \bigl[
          (m_\tau-\mu^*)^2+\big(v_\tau-\sigma^2\big)^2
      \bigr]
      -
      \bigl[
          a_m|m_\tau-\mu^*|
          +a_v\big|v_\tau-\sigma^2\big|
      \bigr]^2
      \\
      &=
      \bigl(a_m^2+a_v^2\bigr)
      \biggl[
          (m_\tau-\mu^*)^2+\big(v_\tau-\sigma^2\big)^2
      \biggr]
      -
      \biggl[
          a_m^2(m_\tau-\mu^*)^2
          +2a_ma_v|m_\tau-\mu^*|\big|v_\tau-\sigma^2\big|
          +a_v^2\big(v_\tau-\sigma^2\big)^2
      \biggr]
      \\
      &=
      a_m^2\big(v_\tau-\sigma^2\big)^2
      -2a_ma_v|m_\tau-\mu^*|\big|v_\tau-\sigma^2\big|
      +a_v^2(m_\tau-\mu^*)^2
      \\
      &=
      \bigl(a_m\big|v_\tau-\sigma^2\big|-a_v|m_\tau-\mu^*|\bigr)^2
      \ge 0.
  \end{align*}
  The first equality is due to the definition of $\overline{\mathcal{R}}_t$. Taking square roots proves~\eqref{eq:moment-to-wealth-stability-2}.

  This completes the proof.

\end{proof}

\subsection{Proof of Theorem~\ref{thm:general}}\label{apx:prf:thm:general}

\begin{proof}

    \textbf{Step 1: relate the trust-update loss to moment error}:
    Given the observed measure $y_{\tau}$ at round $\tau$, Algorithm~\ref{alg:sim2real-e} (line~8) adds $\log \mathcal{N}\big(y_{\tau};\mu_k, \sigma_k^2\big)$ to simulator $k$'s cumulative score. Thus simulator $k$'s one-round trust-update loss is $-\log \mathcal{N}\big(y_{\tau};\mu_k, \sigma_k^2\big)$, and a smaller value leads to more future trust. We have its expecatation as
    \begin{equation}
        \begin{aligned}
            \mathbb{E}\big[(y_{\tau}-\mu_k)^2\big] & = \mathbb{E}\big[(y_{\tau}-\mu^* + \mu^* - \mu_k)^2\big] \\
            & = \mathbb{E}[(y_{\tau}-\mu^*)^2] + 2(\mu^*-\mu_k)\mathbb{E}[y_{\tau}-\mu^*]  +(\mu^*-\mu_k)^2 \\
            & = \sigma^2 + (\mu^*-\mu_k)^2.
        \end{aligned}
    \end{equation}
    The first term is $\sigma^2$ by definition. The middle term is zeros as $\mathbb{E}[y_{\tau}]=\mu^*$. We can then compute simulator $k$'s expected trust-update loss with the loss obtained by the true moment oracle within the same formula as 
    \begin{equation}\label{eq:exp-trust-update-one-round}
        \begin{aligned}
            \mathbb{E} \big[ -\log \mathcal{N}\big(y_{\tau};\mu_k, \sigma_k^2\big) + \log \mathcal{N}\big(y_{\tau};\mu^*, \sigma^2\big) \big] & = \frac{1}{2}\log \frac{\sigma_k^2}{\sigma^2} + \frac{\mathbb{E}\big[(y_{\tau}-\mu_k)^2\big]}{2\sigma_k^2} - \frac{\mathbb{E}\big[(y_{\tau}-\mu^*)^2\big]}{2\sigma^2} \\
            & = \frac{1}{2}\Bigg[ \log\frac{\sigma_k^2}{\sigma^2} + \frac{\sigma^2 + (\mu^*-\mu_k)^2}{\sigma_k^2} - 1 \Bigg].
        \end{aligned}
    \end{equation}

    \textbf{Step 2: bound the trust-weighted loss}:

    Given $y_{\tau} \in [0,1]$, $\mu_k \in [0,1]$, and $v_{\min}\leq \sigma_k^2 \leq v_{\max}$, we have the uniform bound
    \begin{equation}
        \frac{1}{2}\log (2\pi v_{\min}) \leq -\log \mathcal{N}\big(y_{\tau};\mu_k, \sigma_k^2\big) \leq \frac{1}{2}\log (2\pi v_{\max}) + \frac{1}{2v_{\min}}.
    \end{equation}
    Thus the $K$ trust-update losses lie in an interval of length at most $b_{\text{score}}$ at any round. Their negatives, which Algorithm~\ref{alg:sim2real-e} (line~8) accumulates have the same range. 

    Substituting Algorithm~\ref{alg:sim2real-e}'s score update into the exponential sum gives
    \begin{equation}\label{eq:one-round-loss}
        \begin{aligned}
            & \sum_{k=1}^K \exp \big(\eta L_{\tau}^k\big) = \sum_{k=1}^K \exp \big(\eta L_{\tau-1}^k\big) \cdot \exp\big( \eta \log\mathcal{N}\big(y_{\tau};\mu_k, \sigma_k^2\big) \big) \\
            \Rightarrow & 
            \frac{\sum_{k=1}^K \exp \big(\eta L_{\tau}^k\big)}{\sum_{j=1}^K \exp \big(\eta L_{\tau-1}^j\big)} =  \sum_{k=1}^K  \frac{\exp \big(\eta L_{\tau-1}^k\big)}{\sum_{j=1}^K \exp \big(\eta L_{\tau-1}^j\big)} \cdot \exp\big( \eta \log\mathcal{N}\big(y_{\tau};\mu_k, \sigma_k^2\big)\big) = \sum_{k=1}^K   \pi_{\tau}^k \exp\big( \eta \log\mathcal{N}\big(y_{\tau};\mu_k, \sigma_k^2\big)\big) \\
            \Rightarrow & \log \frac{\sum_{k=1}^K \exp \big(\eta L_{\tau}^k\big)}{\sum_{j=1}^K \exp \big(\eta L_{\tau-1}^j\big)}= \log \sum_{k=1}^K   \pi_{\tau}^k \exp\big( \eta \log\mathcal{N}\big(y_{\tau};\mu_k, \sigma_k^2\big)\big)
            \leq  \eta \sum_{k=1}^K\pi_{\tau}^k\log \mathcal{N}\big(y_{\tau};\mu_k, \sigma_k^2\big)+\frac{\eta^2 b_{\text{score}}^2}{8}
        \end{aligned}
    \end{equation}
    The last inequality is by Hoeffding's lemma, and $\pi_{\tau}^k > 0, \sum_{k=1}^K\pi_{\tau}^k=1$.

    Sum \eqref{eq:one-round-loss} from $\tau=1$ through $t$, recall $L_0^k=0$ we have
    \begin{equation}
        \eta \max_{k\in\mathbb{Z}_K} L_t^k - \log K \leq \log \frac{\sum_{k=1}^K \exp\big(\eta L_{t}^k\big)}{K} \leq \eta \sum_{\tau=1}^t \sum_{k=1}^K \pi_{\tau}^k\log \mathcal{N}\big(y_{\tau};\mu_k, \sigma_k^2\big) + \frac{t\eta^2b_{\text{score}}^2}{8}.
    \end{equation}
    The term in the middles is the sum of the L.H.S. of~\eqref{eq:one-round-loss}. We thus have
    \begin{equation}
        \sum_{\tau=1}^t \sum_{k=1}^K \pi_{\tau}^k \big[-\log \mathcal{N}\big(y_{\tau};\mu_k, \sigma_k^2\big)\big] - \min_{k\in\mathbb{Z}_K} \sum_{\tau=1}^t \big[-\log \mathcal{N}\big(y_{\tau};\mu_k, \sigma_k^2\big)\big] \leq \frac{\log K}{\eta}+\frac{t\eta b_{score}^2}{8}.
    \end{equation}
    This compares the cumulative trust-weighted loss with the smallest cumulative loss among the bank of $K$ simulators. We now express both sides relative to the loss obtained from the true moments. At round $\tau$, conditioned on earlier observations, \eqref{eq:exp-trust-update-one-round} gives
    \begin{equation}
        \mathbb{E}\Bigg\{ \sum_{k=1}^K \pi_{\tau}^k \big[ -\log \mathcal{N}\big(y_{\tau};\mu_k, \sigma_k^2\big) + \log \mathcal{N}\big(y_{\tau};\mu^*, \sigma^2\big) \big] \Big| \mathcal{F}_{\tau-1}  \Bigg\} = \sum_{k=1}^K \frac{\pi_{\tau}^k}{2}\Bigg[ \log \frac{\sigma_k^2}{\sigma^2} + \frac{\sigma^2 + (\mu^*-\mu_k)^2}{\sigma_k^2} -1  \Bigg].
    \end{equation}
    For every simulator, the minimum over the bank is no larger than that simulator's cumulative excess loss. As a result,
    \begin{equation}
        \mathbb{E}\Bigg\{ \min_{k\in\mathbb{Z}_K} \sum_{\tau=1}^t \big[ -\log \mathcal{N}\big(y_{\tau};\mu_k, \sigma_k^2\big) + \log \mathcal{N}\big(y_{\tau};\mu^*, \sigma^2\big) \big]  \Bigg\} \leq t\epsilon_K.
    \end{equation}
    Combining the above two we have
    \begin{equation}\label{eq:exp-weighted-argument}
        \frac{1}{t}\sum_{\tau=1}^t \mathbb{E} \Bigg\{ \sum_{k=1}^K \frac{\pi_{\tau}^k}{2} \Bigg[ \log \frac{\sigma_k^2}{\sigma^2} + \frac{\sigma^2 + (\mu^*-\mu_k)^2}{\sigma_k^2} -1  \Bigg]  \Bigg\} \leq \epsilon_K+ \frac{\log K}{\eta t}+\frac{\eta b_{\text{score}}^2}{8} .
    \end{equation}

    \textbf{Step 3: convert~\eqref{eq:exp-weighted-argument} into weighted moment error}:

    Note Lemma~\ref{lma:moment-to-wealth-stability} requires the errors in the weighted mean and variance, while~\eqref{eq:exp-weighted-argument} bounds the excess trust-update loss. The following inequalities supply the conversion:
    \begin{equation}
        \sigma^2 \geq \sigma_k^2 \Rightarrow \frac{\sigma^2}{\sigma_k^2}-1-\log\frac{\sigma^2}{\sigma_k^2} = \int_{1}^{\frac{\sigma^2}{\sigma_k^2}}\frac{u-1}{u}du \geq \frac{\sigma^2}{\sigma_k^2}\int_{1}^{\frac{\sigma^2}{\sigma_k^2}}(u-1)du = \frac{\big( \sigma^2-\sigma_k^2  \big)^2}{2\sigma^2\sigma_k^2} \leq \frac{\big( \sigma^2-\sigma_k^2  \big)^2}{2\sigma^4},
    \end{equation}
    and
    \begin{equation}
        0<\sigma^2\leq\sigma_k^2 \Rightarrow \frac{\sigma^2}{\sigma_k^2}-1-\log\frac{\sigma^2}{\sigma_k^2} = \int^{1}_{\frac{\sigma^2}{\sigma_k^2}}\frac{1-u}{u}du \geq \int^{1}_{\frac{\sigma^2}{\sigma_k^2}}(1-u)du =  \frac{\big( \sigma_k^2-\sigma^2  \big)^2}{2\sigma_k^4}
    \end{equation}
    lead to
    \begin{equation}
        \frac{\sigma^2}{\sigma_k^2}-1-\log\frac{\sigma^2}{\sigma_k^2} \geq \frac{\big( \sigma^2-\sigma_k^2  \big)^2}{2\max\big\{\sigma_k^2, \sigma^2\big\}^2} \geq 0.
    \end{equation}
    Substituting the above into~\eqref{eq:exp-trust-update-one-round}, using $\sigma_k^2 \leq v_{\max}$, $\sigma^2 \leq v_{\max}$, and $c_{\text{mom}}=2v_{\max} + 4v_{\max}^2$, one has
    \begin{equation}
        \begin{aligned}
            & \frac{1}{2}\Bigg[ \log\frac{\sigma_k^2}{\sigma^2} + \frac{\sigma^2 + (\mu^*-\mu_k)^2}{\sigma_k^2} - 1 \Bigg] \geq \frac{(\mu_k-\mu^*)^2}{2v_{\max}} + \frac{\big(\sigma_k^2-\sigma^2\big)^2}{4v_{\max}^2} \\
            \Rightarrow & (\mu_k-\mu^*)^2 + \big(\sigma_k^2-\sigma^2\big)^2 \leq \frac{c_{\text{mom}}}{2}\Bigg[ \log\frac{\sigma_k^2}{\sigma^2} + \frac{\sigma^2 + (\mu^*-\mu_k)^2}{\sigma_k^2} - 1 \Bigg].
        \end{aligned}
    \end{equation}
    Because the trust weights are nonnegative and sum to one, the squared error of either weighted moment is no larger than the trust-weighted average of the corresponding individual squared errors:
    \begin{equation}\label{eq:other-cs}
        \begin{aligned}
            & (m_{\tau}-\mu^*)^2 + \big(v_{\tau}-\sigma^2\big)^2 \leq \sum_{k=1}^K \pi_{\tau}^k \Big[ (\mu_k-\mu^*)^2 + \big(\sigma_k^2-\sigma^2\big)^2 \Big] \\
            \Rightarrow & \frac{1}{t} \sum_{\tau=1}^t \mathbb{E} \Big[ (m_{\tau}-\mu^*)^2 + \big(v_{\tau}-\sigma^2\big)^2 \Big] \leq c_{\text{mom}}\Bigg(\epsilon_K+ \frac{\log K}{\eta t}+\frac{\eta b_{\text{score}}^2}{8}\Bigg).
        \end{aligned}
    \end{equation}

    \textbf{Step 4: apply Lemma~\ref{lma:moment-to-wealth-stability}}:

    \begin{equation}
        \overline{\mathcal{R}}_t \leq c_{\text{stab}} \sum_{\tau=1}^t \sqrt{(m_{\tau}-\mu^*)^2 + (v_{\tau}-\sigma^2)^2} \Rightarrow 
        \frac{\mathbb{E} \Big[\overline{\mathcal{R}}_t\Big]}{t}\leq \frac{c_{\text{stab}}}{t} \sum_{\tau=1}^t \mathbb{E}\bigg[ \sqrt{(m_{\tau}-\mu^*)^2 + (v_{\tau}-\sigma^2)^2}  \bigg].
    \end{equation}
    Applying Jensen's inequality separately at each round leads to
    \begin{equation}
        \begin{aligned}
            \mathbb{E}\bigg[ \sqrt{(m_{\tau}-\mu^*)^2 + (v_{\tau}-\sigma^2)^2}  \bigg] & \leq \sqrt{\mathbb{E}\big[ (m_{\tau}-\mu^*)^2 + (v_{\tau}-\sigma^2)^2  \big]} \\
            \Rightarrow \frac{1}{t} \sum_{\tau=1}^t\mathbb{E}\bigg[ \sqrt{(m_{\tau}-\mu^*)^2 + (v_{\tau}-\sigma^2)^2}  \bigg] & \leq \frac{1}{t} \sum_{\tau=1}^t\sqrt{\mathbb{E}\big[ (m_{\tau}-\mu^*)^2 + (v_{\tau}-\sigma^2)^2  \big]}.
        \end{aligned}
    \end{equation}
    Cauchy-Schwarz applied to the finite sum over rounds gives
    \begin{equation}\label{eq:finite-sum-cs}
        \frac{1}{t}\sum_{\tau=1}^t\sqrt{\mathbb{E}\big[ (m_{\tau}-\mu^*)^2 + (v_{\tau}-\sigma^2)^2  \big]}  \leq \sqrt{\frac{1}{t} \sum_{\tau=1}^t\mathbb{E}\big[ (m_{\tau}-\mu^*)^2 + (v_{\tau}-\sigma^2)^2  \big]}.
    \end{equation}
    Combining~\eqref{eq:finite-sum-cs} and \eqref{eq:other-cs} completes the proof.

\end{proof}

\subsection{Proof for Lemma~\ref{lma:trust-concentration}}\label{apx:prf:lma:trust-concentration}

\begin{proof}
    By Lemma~\ref{lma:trust-concentration}'s given definition:
    \begin{equation}
        r_t^k = \frac{\pi_t^k}{\pi_t^{k'}} = \frac{\exp\big(\eta L_{t-1}^k\big)}{\sum_{j=1}^K\exp\big(\eta L_{t-1}^j\big)} \frac{\sum_{j=1}^K\exp\big(\eta L_{t-1}^j\big)}{\exp\big(\eta L_{t-1}^{k'}\big)} = \frac{\exp\big(\eta L_{t-1}^{k}\big)}{\exp\big(\eta L_{t-1}^{k'}\big)}=\exp\big(\eta\big( L_{t-1}^{k}- L_{t-1}^{k'}\big)\big), 
    \end{equation}
    \begin{equation}
        \begin{aligned}
            G_t^k & = \frac{r_{t+1}^k}{r_t^k} = \exp\Big[ \eta \big( \big(L_t^{k}- L_t^{k'}\big) - \big(L_{t-1}^{k}- L_{t-1}^{k'}\big) \big) \Big]  = \exp\Big[ \eta \big( \big(L_t^{k}- L_{t-1}^{k}\big) - \big(L_t^{k'}- L_{t-1}^{k'}\big) \big) \Big] \\ &= \exp\Big[ \eta \big( \log \mathcal{N}\big(y_t;\mu_k,\sigma_k^2\big) - \log \mathcal{N}\big(y_t;\mu_{k'},\sigma_{k'}^2\big) \big) \Big]
        \end{aligned}
    \end{equation}
    Given $r_t^k$ depends on observation only before round $t$, $G_t^k$ depends on the fresh observation $y_t$, and all samples are i.i.d., we have
    \begin{equation}
        r_{t+1}^k = r_t^kG_t^k \Rightarrow \mathbb{E}\big[r_{t+1}^k\big] = \mathbb{E}\big[r_t^kG_t^k\big] = \mathbb{E}\big[r_t^k\big]\mathbb{E}\big[G_t^k\big]\leq \rho\mathbb{E}\big[r_t^k\big]. 
    \end{equation}
    Hence
    \begin{equation}
        r_1^k=1 \Rightarrow\mathbb{E}\big[r_2^k\big]\leq \rho \Rightarrow \ldots \Rightarrow \mathbb{E}\big[r_t^k\big]\leq \rho^{t-1}.
    \end{equation}
    Moreover, we have
    \begin{equation}
        \begin{aligned}
        \sum_{k\neq k', k\in \mathbb{Z}_K}r_t^k = \sum_{k\neq k', k\in \mathbb{Z}_K} \frac{\pi_t^k}{\pi_t^{k'}} = \frac{1-\pi_t^{k'}}{\pi_t^{k'}} & \Rightarrow 1-\pi_t^{k'} = \pi_t^{k'}\sum_{k\neq k', k\in \mathbb{Z}_K}r_t^k \leq \sum_{k\neq k', k\in \mathbb{Z}_K}r_t^k \\
        & \Rightarrow \mathbb{E}\big[1-\pi_t^{k'}\big] \leq \sum_{k\neq k', k\in \mathbb{Z}_K} \mathbb{E}\big[r_t^k\big] \leq (K-1) \rho^{t-1}.
        \end{aligned}
    \end{equation}
    Thus,
    \begin{equation}
        1-\pi_t^{k'} \leq 1 \Rightarrow \mathbb{E}\big[1-\pi_t^{k'} \big] \leq 1 \Rightarrow \mathbb{E}\big[1-\pi_t^{k'} \big] \leq \min\bigl(1,(K-1) \rho^{t-1}\bigr).
    \end{equation}
    This completes the proof. 
\end{proof}

\subsection{Proof for Theorem~\ref{thm:separation-regret}}\label{apx:prf:lma:trust-separation}

\begin{proof}

\textbf{Step 1: separate reference error from competitor error}:
By definition
\begin{equation}\label{eq:mean-separation}
    m_{\tau}-\mu^* = \sum_{k=1}^K \pi_{\tau}^k \mu_k - \mu^* = \sum_{k=1}^K \pi_{\tau}^k (\mu_k-\mu_{k'}) + \mu_{k'}\sum_{k=1}^K \pi_{\tau}^k - \mu^* = \mu_{k'}-u^* + \sum_{k\neq k'} \pi_{\tau}^k (\mu_k-\mu_{k'}),
\end{equation}
\begin{equation}\label{eq:var-separation}
    v_{\tau} - \sigma^2 = \sum_{k=1}^K \pi_{\tau}^k \sigma_k^2 - \sigma^2 = \sum_{k=1}^K \pi_{\tau}^k \big(\sigma^2_k-\sigma^2_{k'}\big) + \sigma^2_{k'}\sum_{k=1}^K \pi_{\tau}^k - \sigma^2  =\sigma^2_{k'}-\sigma{^2} + \sum_{k\neq k'} \pi_{\tau}^k \big(\sigma_k^2-\sigma_{k'}^2\big).
\end{equation}

\textbf{Step 2: bound the competitor contribution}:
Given \eqref{eq:mean-separation}, \eqref{eq:var-separation}, and the definitions of $h_m$, $h_v$, $D_m$, and $D_v$, apply the triangle inequality,
\begin{equation}
    |m_{\tau}-\mu^*| \leq h_m + \sum_{k\neq k'} \pi_{\tau}^k |\mu_k-\mu_{k'}| \leq h_m+D_m\sum_{k\neq k'} \pi_{\tau}^k = h_m+D_m\big(1-\pi_{\tau}^{k'}\big),
\end{equation}
\begin{equation}
    \big|v_{\tau} - \sigma^2\big| \leq h_v+\sum_{k\neq k'} \pi_{\tau}^k \big|\sigma_k^2-\sigma_{k'}^2\big|\leq h_v+D_v\sum_{k\neq k'}\pi_{\tau}^k = h_v+D_v\big(1-\pi_{\tau}^{k'}\big).
\end{equation}
Apply Lemma~\ref{lma:trust-concentration} and use the definition of $S_t$, we have
\begin{equation}\label{eq:Emdiff}
    \mathbb{E}|m_{\tau}-\mu^*| \leq h_m + D_m \mathbb{E}\big[1-\pi_{\tau}^{k'}\big] \leq h_m + D_m \min\big\{ 1,(K-1)\rho^{\tau-1} \big\} \Rightarrow \frac{1}{t}\sum_{\tau=1}^t \mathbb{E}|m_{\tau}-\mu^*| \leq  h_m+ \frac{D_m S_t}{t}.
\end{equation}
\begin{equation}\label{eq:Evdiff}
    \mathbb{E}\big|v_{\tau} - \sigma^2\big| \leq h_v + D_v \mathbb{E}\big[1-\pi_{\tau}^{k'}\big] \leq h_v + D_v \min\big\{ 1,(K-1)\rho^{\tau-1}\big \} \Rightarrow \frac{1}{t}\sum_{\tau=1}^t \mathbb{E}\big|v_{\tau} - \sigma^2\big| \leq h_v+\frac{D_v S_t}{t}
\end{equation}

\textbf{Step 3: apply Lemma~\ref{lma:moment-to-wealth-stability}}:
\begin{equation}
    \begin{aligned}
    & \overline{\mathcal{R}}_t = \sup_{\mu\in\mathcal{M}} \mathcal{R}_t^+(\mu) \leq \sum_{\tau=1}^t \big( a_m|m_{\tau}-\mu^*| + a_v\big|v_{\tau} - \sigma^2\big| \big) \\
    \Rightarrow & \frac{\mathbb{E}\big[ \overline{\mathcal{R}}_t \big] }{t} \leq \frac{a_m}{t}\sum_{\tau=1}^t\mathbb{E}|m_{\tau}-\mu^*| + \frac{a_v}{t}\sum_{\tau=1}^t\mathbb{E}\big|v_{\tau} - \sigma^2\big| \\
    \overset{\eqref{eq:Emdiff},\eqref{eq:Evdiff}}{\Rightarrow} & \frac{\mathbb{E}\big[ \overline{\mathcal{R}}_t \big] }{t} \leq a_m\Bigg(h_m+ \frac{D_m S_t}{t}\Bigg)+a_v\Bigg(h_v+\frac{D_v S_t}{t}\Bigg) .
    \end{aligned}
\end{equation}
This proves \eqref{eq:seperation-regret}.

\textbf{Step 4: Bound $S_t$}:
By the definition of $S_t$ (let $\left\lceil \cdot \right\rceil$ denote the ceiling function returning the smallest integer greater than or equal to the enclosed value):
\begin{equation}
    \begin{aligned}
    S_t & = \sum_{\tau=1}^t\min\big\{ 1,(K-1)\rho^{\tau-1}\big\}  \leq \sum_{\tau=1}^{1+ \left\lceil \frac{\log(K-1)}{-\log \rho} \right\rceil } 1 + \sum_{\tau=2+\left\lceil \frac{\log(K-1)}{-\log \rho} \right\rceil}^{\infty}(K-1)\rho^{\tau-1} \\
    & = 1 + \left\lceil \frac{\log(K-1)}{-\log \rho} \right\rceil + \frac{(K-1)\rho^{1+\left\lceil \frac{\log(K-1)}{-\log \rho} \right\rceil} }{1-\rho} \\
    & \textbf{Note: } \left\lceil \frac{\log(K-1)}{-\log \rho} \right\rceil \geq \frac{\log(K-1)}{-\log \rho} \Rightarrow (K-1)\rho^{1+\left\lceil \frac{\log(K-1)}{-\log \rho} \right\rceil} \leq \rho \\
    & \Rightarrow S_t \leq 1 + \left\lceil \frac{\log(K-1)}{-\log \rho} \right\rceil + \frac{\rho}{1-\rho} \leq 1 +\left\lceil \frac{\log(K-1)}{-\log \rho} \right\rceil + \frac{1}{1-\rho},
    \end{aligned}
\end{equation}
which proves~\eqref{eq:S_t-bound}.
This completes the proof.

\end{proof}

\section{Experiment Details}\label{apx:exp}
This section provides more details (particularly on \textbf{C2} and \textbf{C3}) complementing the main paper's empirical results in Section~\ref{sec:exp} as well as the released code base supporting the reproducibility of the presented results.

\subsection{\textbf{C2} - command tracking}
The real-world samples are collected using a customized software bridge that translates joystick actions from a Logitech F310 gamepad into GR00T compatible commands $x^c$ (linear and yaw) at 50 Hz. For each sample, the velocity error is computed as the difference between the command velocity sent through the gamepad and the actual velocity measured by the OptiTrack Motion Capture System (OTS). The lower body controller runs at a fixed control frequency (50 Hz), and the velocity measurements are sampled at the same frequency. 

To establish the paired correspondence between the real-world and simulation data required by SureSim, we replay the same commands recorded from the real-world samples in MuJoCo environment. The resulting paired velocity error is computed from the difference between the replayed command and the measured robot state $x^r$. For real-world samples, $x^r$ is obtained from the Motion Capture System, while for simulated samples, $x^r$ is obtained from the corresponding robot state in MuJoCo. For the \texttt{Sim\_MuJoCo\_3}, we construct a diverse simulation bank by varying the physical parameters of the MuJoCo environment. 

In addition, the implementation of SureSim is adapted from the opensource code~\footnote{\href{https://github.com/irom-princeton/rapid-policy-evaluation}{https://github.com/irom-princeton/rapid-policy-evaluation}}. Details of the adaptation can be found in the released code base.

\subsection{\textbf{C3} - peg-in-hole}
Continuous mobile manipulator benefits for large, curved, and complex workpieces in many industrial and unstructured environments. However, such flexibility also introduces additional source of performance uncertainty, preventing stringent pose repeatability and accuracy. To identify and quantify this uncertainty, the Configurable Mobile Manipulator Apparatus (CMMA) was developed in~\cite{nist24continuous} to provide a standardized set of target locations for evaluating the performance of the mobile manipulator during continuous motion. As illustrated in Fig.~\ref{fig:setup}, the Type B CMMA used in this study consists of two sides, each containing six retro-reflective fiducials of different sizes. The mobile manipulator system consists of an autonomous mobile robot cart transporter (AMR-CT) and a manipulator, with a Retro-reflective Laser Sensor and Emitter (RLS) mounted at the end effector of the manipulator to detect and intercept these fiducials as the mobile manipulator moves along the apparatus. The AMR-CT uses manufacturer-provided mapping and navigation software together with custom ROS-based velocity control to drive the platform along the CMMA. 

\begin{figure*}[t]
    \centering
    \includegraphics[width=\linewidth,trim=0cm 3.5cm 4cm 3.5cm,clip]{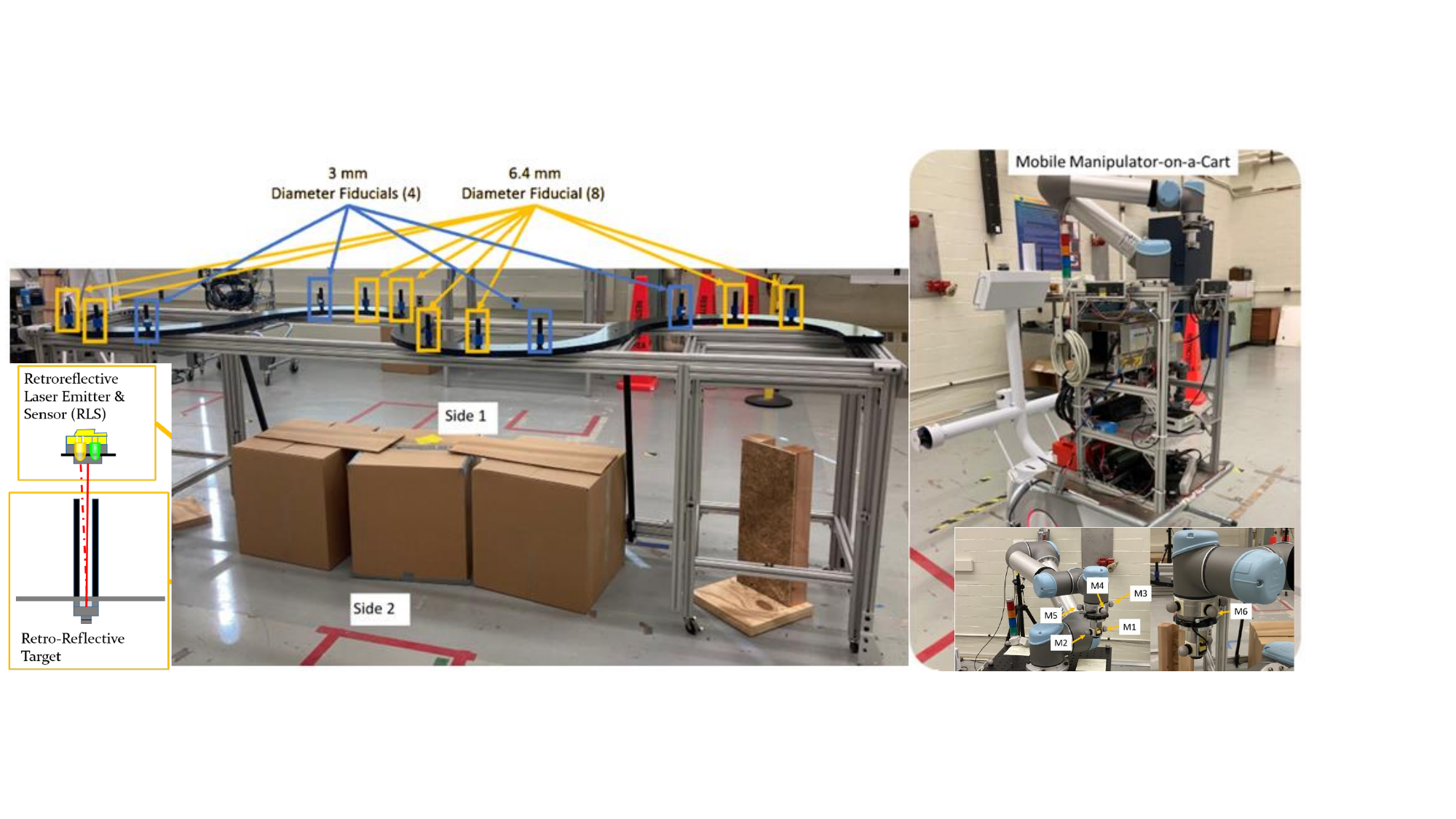}
    \caption{Overview of the experimental setup. The left side shows the Type B CMMA, while the right side shows the mobile manipulator mounted on the AMR-CT. The left corner highlights the Retro-reflective Laser Sensor and Emitter (RLS), and the right corner shows the optical tracking system (OTS) markers mounted on the end effector of the manipulator. \textit{Note: the figure is adapted from~\cite{nist24continuous}}}
    \vspace{-0.5cm}
    \label{fig:setup}
\end{figure*}

In this experiment, the fiducial-interception task serves as a standardized surrogate for a peg-in-hole-type manipulation task that the end-effector must continuously locate and align with a designed target and successfully intercept it while both the AMR-CT and the manipulator are in motion. The experiments design consisted of a full $2^3$ factorial experiment with three factors: (i) the search algorithm (varied between Deterministic Spiral Search and a stochastic Unscented Kalman Filter (UKF)-based search), (ii) the mobile-base translation speed (either $0.01~\mathrm{m/s}$ or $0.025~\mathrm{m/s}$), and (iii) the CMMA side (either Side 1 or Side 2). The order of the eight factor combinations was randomized within each replicate to reduce potential ordering effects. The original NIST study analyzed the experimental data primarily using a $2^3$ factorial Analysis of Variance, using p-values to assess the statistical significance of the metrics, along with other prior analysis. Two versions of the NIST dataset are publicly available, released in February\footnote{\href{https://data.nist.gov/od/id/mds2-3061}{https://data.nist.gov/od/id/mds2-3061}} and June\footnote{\href{https://data.nist.gov/od/id/mds2-3187}{https://data.nist.gov/od/id/mds2-3187}}, respectively. In this work, we use the June dataset which has six replicate trials for each of the eight experimental conditions, resulting in a total of $6\times8=48$ experimental runs.

The performance in this report is evaluated through two primary measurements: (i) the fiducial \texttt{interception rate}, defined as the percentage of target fiducials successfully intercepted during one run and (ii) the average fiducial \texttt{search time}, defined as the mean time required to locate and intercept an individual target, capped at the maximum time allotted to search for that target before the vehicle carries it out of range. In addition, we derived an additional continuous measurement from the recorded data: the distance between the fiducial target and end manipulator end-effector measured by OTS. Details of the computation and data processing procedure are provided in the code base. 

\subsection{\textbf{C3} - pushover}
The pushover stability test of the Unitree Go2 robot follows a standard testing method currently being developed as part of WK86916 by the subcommittee ASTM F45.06 Legged Robot Systems. 

The impactor is a 50 mm-diameter, 500 g steel ball complies with GB4943, GB4906, IEC60335, IEC60950, IEC60598, UL507, concerning the impact resistance performance of various electronic appliances. It swung as a pendulum from a constant vertical height to impact the enclosure of the subject robot. The subject is a Unitree Go2 EDU quadruped robot with commercial software stack capable of in-place standing (undisturbed) (firmware version anonymized). Across all impact tests, the subject robot starts from the same stationary foot placement ($25 \times 40$ cm) in contact with the ground. However, its body pose still varies due to the intrinsic design of the robot controller. The impact area therefore still differs from trial to trial, which contributes in part to the uncertainty of the performance measure.

\printnomenclature

\end{document}